\documentclass[11pt]{article}

\usepackage{fullpage}
\usepackage{amsmath}
\usepackage{amsthm,amsfonts,amssymb}
\usepackage{xspace}
\usepackage{bbm}
\usepackage{subcaption}
\usepackage{graphicx}
\usepackage{bbm}
\usepackage{hyperref}   
\usepackage{etoolbox}% http://ctan.org/pkg/etoolbox
\usepackage{enumitem}
\usepackage{authblk}

\usepackage[noend]{algorithmic}
\usepackage[ruled,vlined]{algorithm2e}
\usepackage{comment}
\usepackage{natbib}
\usepackage{float}

	\hypersetup{
			colorlinks=true,
			linkcolor=[rgb]{0,0,.5},
			urlcolor=[rgb]{0,0,.5},
			citecolor=[rgb]{0,0,.5},
			pdfstartview=FitH}

\usepackage{accents}

\usepackage{graphicx} % Allows including images
\usepackage{booktabs} % Allows the use of \toprule, \midrule and \bottomrule in tables

\usepackage{caption}

\usepackage{xcolor}

\usepackage{pgfplots}

\pgfplotsset{compat=1.10}

\usepackage{tikz}

\usepackage{enumitem}
\usepackage{mathtools}
\usetikzlibrary{quotes,positioning}
\allowdisplaybreaks

\newtheorem{theorem}{Theorem}[section]

\newtheorem{definition}{Definition}[section]

\newtheorem{example}{Example}[section]

\newcounter{numrellocal}% Local counter for numering relations
\renewcommand{\thenumrellocal}{\arabic{numrellocal}}% Counter numrellocal uses lowercase roman numerals
\newcounter{numrelglobal}% Global counter for numering relations
\makeatletter
\newcommand{\numrel}[2]{% Relation numbering
  \stepcounter{numrellocal}% Increment local counter
  \refstepcounter{numrelglobal}% Increment global counter and create correct reference hook, with label-text from local counter
  \ltx@label{#2}% Label numrel counter
  \overset{(\thenumrellocal)}{#1}% Print counter + relation
}
\makeatother
\usepackage{xcolor}

\newcommand{\p}{\mathbb{P}}
\newcommand{\E}{\mathop{\mathbb{E}}}

\newcommand{\cX}{\mathcal{X}}
\newcommand{\cY}{\mathcal{Y}}

\newcommand{\cG}{\mathcal{G}}

\newcommand{\cD}{\mathcal{D}}

\newcommand{\one}{\mathbbm{1}}

\usepackage{thmtools} 
\usepackage{thm-restate}

\usepackage{amsmath}
\usepackage{pgfplots}

\title{The Relationship between No-Regret Learning and Online Conformal Prediction
}

\author[1]{Ramya Ramalingam}
\author[1]{Shayan Kiyani}
\author[1]{Aaron Roth}
\affil[1]{Department of Computer and Information Sciences, University of Pennsylvania}

\begin{document}
\maketitle

\begin{abstract}
Existing algorithms for online conformal prediction---guaranteeing marginal coverage in  adversarial settings---are variants of online gradient descent (OGD), but their analyses of worst-case coverage do not follow from the regret guarantee of OGD. What is the relationship between no-regret learning and online conformal prediction? We observe that although standard regret guarantees imply marginal coverage in i.i.d. settings, this connection fails as soon as we either move to adversarial environments or ask for group conditional coverage. On the other hand, we show a tight connection between \emph{threshold calibrated} coverage and swap-regret in adversarial settings, which extends to group-conditional (multi-valid) coverage. We also show that algorithms in the \emph{follow the perturbed leader} family of no regret learning algorithms (which includes online gradient descent) can be used to give group-conditional coverage guarantees in adversarial settings for arbitrary grouping functions. Via this connection we analyze and conduct experiments using a multi-group generalization of the ACI algorithm of \citet{gibbs2021adaptive}. 
\end{abstract}

\section{Introduction}
\label{sec:intro}

In  prediction problems over a label space $\cY$, a popular method for quantifying uncertainty is to produce \emph{prediction sets} $C(x) \subseteq \cY$ that contain subsets of the label space. Given features $x$, the intended semantics of $C(x)$ is that the true label $y$ will fall into the prediction set (i.e. will be \emph{covered} by the prediction set) with some specified probability, say $90\%$. A-priori producing prediction sets is a very high dimensional problem: there are $2^{|\cY|}$ possible prediction sets, which becomes intractable to enumerate over for even moderately large label spaces. But a key insight of the conformal prediction literature (see e.g. \cite{angelopoulos2021gentle}) is that given an arbitrary \textit{non-conformity score} $f : \cX \times \mathcal{Y} \to \mathbb{R}_{\geq 0}$, there is a 1-dimensional family of nested prediction sets defined as $C_\tau(x) = \{y \in \cY : f(x, y) \leq \tau\}$ that we can optimize over, and --- simply by adjusting $\tau$, we can obtain \emph{marginal} coverage at any desired rate $q \in (0,1)$. If there is a data distribution $\cD$, marginal coverage at a rate of $q$ corresponds to the guarantee that $\Pr_{(x,y) \sim \cD}[y \in C_\tau(x)] = q$. Stronger guarantees of \emph{group} conditional coverage and (threshold-calibrated) ``multivalid'' coverage have also been recently developed \cite{jung2021moment,gupta2022online,bastanipractical,jung2023batch,noarov2023high,gibbs2023conformal}, which ask for coverage to hold conditionally on various events. These methods involve learning a threshold \emph{model} $\tau:\cX\rightarrow \mathbb{R}$, and producing prediction sets of the form $C_\tau(x) = \{y : f(x) \leq \tau(x)\}$. Group conditional coverage starts with a collection of groups $g_1,\ldots,g_k$ represented as indicator functions $g_i:\cX\rightarrow \{0,1\}$ (which can be arbitrary and intersecting) and asks that $\E_{(x,y) \sim \cD}[y \in C_\tau(x) | g_i(x)=1] = q$ for each group $g_i$. Multivalid coverage asks that the prediction sets simultaneously satisfy group conditional coverage while also being \emph{threshold calibrated} --- i.e. it further conditions on the threshold value $\tau(x) = v$ (in a manner similar to calibration), and asks that for all groups $g_i$ and all threshold values $v$: $\E_{(x,y) \sim \cD}[y \in C_\tau(x) | g_i(x)=1, \tau(x) = v] = q$. Conditioning on the threshold value prevents coverage from being obtained in an uninformative way by predictors which ``hedge'' by mixing over different threshold values with very different coverage rates \citep{bastanipractical,jung2023batch}.
The algorithm proposed by \citet{jung2023batch} for obtaining group conditional coverage over a set of groups learns $\tau(x)$ by minimizing pinball loss over the class of linear combinations of group indicator functions. 

It is also possible to obtain coverage guarantees in online adversarial settings in which there is no distribution $\cD$, but instead an arbitrary sequence of examples $(x_t,y_t)$ that arrive sequentially. Here we ask for empirical coverage --- the threshold (or function) $\tau_t$ can now be updated over time, and marginal coverage over $T$ rounds corresponds to the requirement that $1/T \sum_{t=1}^T \mathbf{1}[y_t \in C_{\tau_t}(x_t)] = q$. Group conditional and multivalid coverage can be similarly defined in the sequential setting. \citet{gibbs2021adaptive} and \citet{gupta2022online} independently studied coverage in online adversarial settings and proposed quite different algorithms. \citet{gibbs2021adaptive} give a very lightweight algorithm which implements online gradient descent on the pinball loss, and prove that it guarantees marginal coverage. This algorithm chooses $\tau_t$ every day independently of the context $x_t$. They give a custom analysis of the coverage properties of their algorithm, however, and do not derive them from the regret guarantees of online gradient descent---i.e. the guarantee that the cumulative pinball loss of online gradient descent is no larger than the best single threshold would have obtained in hindsight.  
\citet{gupta2022online} (later refined by \citet{bastanipractical}) on the other hand give a more complex algorithm modeled on techniques for sequential calibration, and prove that it obtains multivalid coverage with respect to an arbitrary collection of group functions. As it promises group conditional coverage, it necessarily chooses $\tau_t$ as a function of $x_t$. This suggests a number of interesting questions:
\begin{enumerate}
    \item If a sequence of thresholds $\tau_t$ has no \emph{regret} with respect to the pinball loss, does this on its own guarantee coverage? Are there circumstances in which it does? Is there a stronger form of regret that does?
    \item Can algorithms for guaranteeing \emph{group conditional} regret with respect to pinball loss (e.g. \cite{blum2020advancing,lee2022online}) similarly be used to give group conditional or multivalid coverage guarantees? 
    \item Alternately, if the guarantees of \citet{gibbs2021adaptive} do not follow from the regret guarantee of online gradient descent, can we identify a broader class of algorithms that offer these guarantees and generalize them to offer group conditional (rather than just marginal) coverage guarantees?
\end{enumerate}

In this paper we provide answers to these questions.
\subsection{Our Results}
\label{ref:subsec:def}
\subsubsection{Regret and Coverage}
First we consider the relationship between different kinds of regret that a sequence of thresholds $\tau_t$ can have with respect to the pinball loss objective, and how they correspond to coverage guarantees. 
\paragraph{External Regret} A sequence of thresholds $\tau_1,\ldots,\tau_T$ is said to have no \emph{external} regret if its average pinball loss is no larger than that of the best fixed threshold $\tau^*$ chosen in hindsight. This is the kind of regret guarantee offered by algorithms like online gradient descent \citep{zinkevich2007regret} and multiplicative weights \citep{arora2012multiplicative}. We observe  that in adversarial settings, a no external regret guarantee on the thresholds $\tau_t$ does \emph{not} guarantee non-trivial coverage (similar observations have previously been made \citep{gibbs2021adaptive}), but show that it does if the algorithm chooses its threshold independently of any context $x_t$ (as ACI does) and outcomes $y_t$ are drawn i.i.d. from an unknown distribution $\cD$. We then turn our attention to \emph{group conditional} regret guarantees. A group conditional external regret guarantee promises no external regret not just marginally over the whole sequence of rounds $\{1,\ldots,T\}$, but simultaneously on each subsequence $S(g_i) = \{t : g_i(x_t)=1\}$ corresponding to rounds on which the examples are members of group $g$. Algorithms promising no group conditional regret (such as \cite{blum2020advancing, acharyaoracle}) must receive \emph{context} $x_t$ at each round before they make their prediction specifying which groups the current example is a member of. We show that even when the examples $(x_t,y_t)$ are drawn i.i.d. from a distribution $\cD$, contextual algorithms obtaining no external regret (and hence any algorithm obtaining no group conditional external regret) do not necessarily obtain any non-trivial  coverage bounds --- because even in i.i.d. settings, the context can correlate the prediction and the outcome.
\paragraph{Swap Regret} We then turn our attention to \emph{swap regret}, which corresponds to a guarantee of no external regret conditional on the value of the threshold played --- i.e.  a no swap regret guarantee corresponds to a guarantee of no external regret simultaneously on each subsequence $S(v) = \{t : \tau_t = v\}$ defined by threshold values $v$. There exist many efficient algorithms for guaranteeing no swap regret for convex losses (and it has a close connection to the computation of correlated equilibrium in game theory) \cite{blum2007external,foster1999regret,dagan2024external,peng2024fast}. There also exist efficient algorithms for obtaining \emph{group-conditional} swap regret for arbitrary polynomially sized collections of intersecting groups \cite{lee2022online,noarov2023high}. We show that (under mild smoothness assumptions on the distribution), threshold calibrated coverage is equivalent to swap regret in the sense that any algorithm for guaranteeing no swap regret with respect to the pinball loss produces thresholds that guarantee threshold calibrated coverage at the target rate, and vice versa. This tight connection carries over to group-conditional swap regret --- group conditional swap regret is equivalent to multivalid coverage over the same group structure. This connection holds even for algorithms that use context. This gives new algorithms for guaranteeing group conditional multivalid coverage.

\subsubsection{Coverage Guarantees Beyond Regret}
We then turn our attention to generalizations of the ``ACI'' guarantee that \citet{gibbs2021adaptive} prove for online gradient descent on the (1 dimensional) pinball loss. \citet{gibbs2021adaptive} analyze their algorithm by showing that 1) the marginal mis-coverage rate is proportional to the magnitude of the threshold used at the final iterate, and that 2) all iterates (and so in particular the final one) are bounded. We generalize this result in two ways. First, given a collection of groups $\cG$, we consider multi-dimensional problems in which we minimize the pinball loss of a function $\tau_t(x) = \langle \theta^t, g(x_t) \rangle$ defined as a $|\cG|$-dimensional linear function of the group indicator functions (mirroring the form of $\tau(x)$ used to obtain group conditional coverage in batch settings in \cite{jung2023batch}).  We show that if we optimize the pinball loss of $\tau_t(x)$ using any algorithm from the ``follow the regularized leader'' (FTRL) family of no-regret algorithms \cite{shalev2012online} (a family that includes online gradient descent, but also multiplicative weights and many other no regret learning algorithms), then the coverage rate within each group $g_i$ can be bounded as a function of the magnitude of $\theta^T_i$ (the coordinate of the parameter vector corresponding to group $i$) and the gradient of the regularization function used to instantiate FTRL. This generalizes the bound proven in \cite{gibbs2021adaptive} for 1-dimensional online gradient descent (which is an instance of FTRL regularized by the Euclidean norm). We then prove that when using $|\cG|$-dimensional online gradient descent for group conditional coverage, it is possible to bound the magnitude of the maximum coordinate of $\theta^T$ by $O(\sqrt{T})$ even when the group functions need not be binary, and can be general weighting functions $g_i:\cX \rightarrow [0,1]$. This implies a $O(\sqrt{T})$ group conditional coverage bound. We show that this is tight (even in 1-dimension) for \emph{real valued} weighting functions by demonstrating an $\Omega(\sqrt{T})$ lower bound. Finally, we perform an experimental evaluation of this algorithm, and compare it to the online algorithm for guaranteeing multivalid coverage given by \cite{bastanipractical}. We show that our method converges faster to the desired coverage rate. Further, though our upper-bound on the rate of the maximum coordinate of $\theta_T$ grows with $T$, empirically we see in each experiment (using binary groups) that it grows much slower and remains very small over the full transcript. We conjecture (but cannot prove) that for binary groups,  the norm of $\theta_T$ can be bounded by a much more slowly growing function of $T$ (or perhaps can be bounded only as a function of $k$, the number of groups, independently of $T$).

% \ar{Are we going to say something for bias/squared loss, and more generally for elicitable properties?}

\subsection{Additional Related Work}
Online conformal prediction was introduced by \citet{gibbs2021adaptive}, who gave the ``ACI'' (Adaptive Conformal Inference) algorithm, and noted that it was an instantiation of 1-dimensional online gradient descent on the pinball loss --- but that the coverage bound did not follow from the standard regret analysis of online gradient descent. This spurred a number of follow up works that modified or refined the original ACI analysis \citep{gibbs2022conformal,feldman2022achieving,lekeufack2024conformal,angelopoulos2024online,bhatnagar2023improved}, some of them by making explicit connections to algorithms which guarantee more refined adaptive regret bounds \citep{gibbs2022conformal,bhatnagar2023improved} --- but the worst-case coverage bounds are never derived via the regret bounds, which are used to make auxiliary claims (such as convergence to the true quantile of the loss in stationary or slowly changing environments). 

In a parallel line of work, \citet{gupta2022online} introduced the problem of online uncertainty quantification in the form of mean, variance, and quantile estimation, using techniques deriving from the online calibration literature \citep{foster1998asymptotic}. \citet{bastanipractical} gave a refinement of their quantile calibration technique to give an online conformal prediction method that gave conditional guarantees of various sorts. Coverage bounds from algorithms of this sort follow from quantile-calibration arguments. 

\cite{romano2020malice} introduced the problem of group conditional coverage and studied it for disjoint groups. \cite{foygel2021limits} consider intersecting groups and propose running separate algorithms for each group, and for examples that are in multiple groups, using the most conservative threshold amongst each of the group-specific  algorithms. \cite{jung2021moment} give the first non-conservative method for getting group conditional coverage for intersecting groups, by adapting ideas from multicalibration \citep{hebert2018multicalibration} to calibrate to \emph{moments} of the score function, conditional on group membership. \cite{gupta2022online} give algorithms for group-conditional quantile multicalibration, and show how this can be used to  give tight ``multivalid'' confidence intervals. \cite{bastanipractical} and \cite{jung2023batch} apply these ideas explicitly to conformal prediction. \cite{deng2023happymap} also show how to generalize multicalibration to give group conditional guarantees in conformal prediction. \cite{gibbs2023conformal} give a variant of the algorithm from \cite{jung2023batch} which gives coverage guarantees in expectation over the calibration set, rather than PAC-style guarantees as in \cite{jung2023batch}. \cite{noarov2023high} give online algorithms using ideas from high dimensional calibration that are able to produce prediction sets whose validity holds subject to any set of conditioning events known at the time of prediction --- this includes group conditional validity, but also prediction-set size conditional validity, action conditional validity, among many other things. 

The characterization we give of threshold calibrated coverage by swap regret bounds on the pinball loss mirrors an equivalence between swap regret on the squared loss and (mean) calibration \cite{foster1998asymptotic,foster1999regret}. More generally the connection between calibration of different distributional quantities and their corresponding ``elicitation functions''  was made by \cite{NR23}.

Our generalization of the 1-dimensional ACI bounds to group conditional coverage bounds was independently and concurrently discovered by \citet{angelopoulos2025gradient}. \citet{angelopoulos2025gradient} develop their bounds as part of an elegant and  general theory of \emph{gradient equilibrium}, which they show is neither implied by nor implies external regret bounds, whereas we restrict attention to pinball loss and the online coverage problem. We generalize online gradient descent to algorithms derived as instantiations of follow the regularized leader, whereas they give a generalization to algorithms derived as an instantiation of proximal mirror descent.

\iffalse
Given a label space $\mathcal{Y}$, one of the focuses of uncertainty quantification is on producing prediction sets in $2^{\mathcal{Y}}$ that cover true labels with high probability. In general, the set of prediction sets may be complex. Conformal prediction is a form of uncertainty quantification that uses something called a \textit{non-conformity score} $f : \mathcal{Y} \times \mathcal{Y} \to \mathbb{R}_{\geq 0}$ to reduce this collection to a nested set parametrized by a single variable (call it $x$), in the following manner:
\begin{equation}
    C(t, x) = \{y \in \mathcal{Y}: f(t, y) \leq x\}
\label{eq:pred-set}
\end{equation}
This paper aims to explore the connections between conformal-style prediction algorithms (which in an abstract sense, operate as one-dimensional predictive algorithms that achieve coverage guarantees) and no-regret algorithms, which are predictive algorithms that achieve performance comparative to some fixed comparator class of predictors. 
\fi

\section{Definitions}
\label{sec:def}
Define a joint feature-label space $(\mathcal{X}, \mathcal{Y})$. In uncertainty quantification, one of our goals is to learn prediction sets $C: \mathcal{X} \to 2^{\mathcal{Y}}$ that satisfy certain probabalistic guarantees. Specifically, for some specified coverage rate $q$, we would like to produce sets that include the true label with probability $q$. Conformal prediction simplifies this problem by defining a collection of nested sets parametrized by a single variable (call it $\tau$) in the following manner:
\begin{equation}
    C_{\tau}(x) = \{y \in \mathcal{Y}: f(x, y) \leq \tau\}
\end{equation}
where $f: \mathcal{X} \times \mathcal{Y} \to [0,1]$ can be any arbitrary function, called a \textit{non-conformity score}. Fix a feature-label pair $(x,y)$. Given $x$, to choose a prediction set from this collection that includes the true label $y$, we must choose some $\tau \geq f(x,y)$.
Conformal methods can be viewed as prediction problems, to choose the ``correct" value of $\tau$--- a target quantile of the distribution on scores $f(x,y)$. In distributional settings, to achieve an exact coverage guarantee of the form $\mathbb{P}_{(x,y) \sim \mathcal{D}}[y \in C_\tau(x)] \approx q$, this can be done by using training data to get an estimate $\hat{\tau}$ of the $q$-th quantile of non-conformity score values induced by $\mathcal{D}$. 

The $q$-th quantile of a distribution minimizes the expectation of a convex function called the \textit{pinball loss}, defined as: 
\[
p_{q}(\hat{\tau}, \tau)=\begin{cases} 
      q(\tau - \hat{\tau}) & \text{ if } \tau \geq \hat{\tau} \\
      (q - 1)(\tau - \hat{\tau})& \text{ if } \tau <  \hat{\tau}
    \end{cases}
\]
The procedure used in conformal prediction can therefore also be seen as finding an estimate $\hat{\tau}$ of the true value $\tau$ that minimizes the expected pinball loss. But in the adversarial setting, there is no longer any distribution over which to estimate a fixed parameter $\hat{\tau}$. Instead, at each round $t$, we may be  given features $x_t$ (if we are in the ``contextual'' setting) and use it to predict a parameter $\hat{\tau}_t$ (and correspondingly the prediction set $C_{\hat{\tau}_t}$). In the non-contextual setting we must choose $\hat{\tau}_t$ solely based on the history thus far. Then we receive the true label $y_t$. Note that $y_t \in C_{\hat{\tau}_t}$ iff $\hat{\tau}_t \geq \tau_t$. Thus, we may view online conformal prediction as a sequential prediction task, where over $T$ rounds,
\begin{enumerate}
    \item The adversary chooses a joint distribution over contexts $x_t \in \cX$ and non-conformity score thresholds $\tau_t \in [0,1]$.
    \item The learner, given a realized context $x_t$, makes a prediction $\hat{\tau}_t$ of the score threshold.
    \item The learner receives a realized threshold $\tau_t$. 
\end{enumerate}
Given a desired coverage level $q$, the goal is to make predictions such that $\frac{1}{T} \sum_{t=1}^T \one[\hat{\tau}_t \geq \tau_t] \approx q$. Note that for simplicity, we abstract away the true label $y_t$ and non-conformity score $f_t$ here. Implicitly, $\tau_t = f_t(x_t, y_t)$. We assume also that all thresholds $\tau_t$ are bounded in $[0,1]$. In practice, non-conformity scores can be normalized to ensure this holds. 
\begin{definition}[Transcript]
A \emph{transcript} $\Pi_T = \{(x_t, \tau_t, \hat{\tau}_t)\}_{t=1}^T$ denotes a sequence of contexts, outcomes and predictions in the sequential prediction setting. Let $\Pi^* = (\mathcal{X} \times [0,1] \times [0,1])^*$ denote the set of all transcripts. 
\end{definition}
%One might expect that, analogous to the distributional setting, to achieve marginal coverage it is sufficient to produce a sequence of predictions that produces a sum of pinball losses $p_{q}$ that is competitive with the best fixed parameter value in hindsight (the $q$-th quantile over the set $\{\tau_t\}_{t=1}^T$), i.e. has external regret with respect to the pinball loss. As we will see, this is not the case; in adversarial settings, one can achieve external regret with respect to pinball loss without getting the desired marginal coverage, and only a stronger kind of regret guarantees coverage.
\subsection{Coverage}

\begin{definition}[Coverage, Coverage Error]
    Given a transcript $\Pi_T = \{(x_t, \tau_t, \hat{\tau}_t)\}_{t=1}^T$, the \emph{coverage} of $\Pi_T$ is defined as:
    \begin{equation*}
        \texttt{Cov}(\Pi_T) = \frac{1}{T} \sum_{t=1}^T \textbf{1}[\hat{\tau_t} \geq \tau_t]
    \end{equation*}
    For a desired coverage rate $q \in (0,1)$, we have \textit{coverage error} $\gamma$ with respect to $q$ if $|\texttt{Cov}(\Pi_T) - q| \leq \gamma$. 
\end{definition}
One can examine coverage not just marginally over the transcript, but also over subsequences of the transcript that define groups within the full sequence. These groups may be defined by context, the predicted threshold, or even the past transcript, as long as membership can be determined at the start of the round, as the learner makes a prediction. 
\begin{definition}[Group]
    A group $G: \Pi^* \times \mathcal{X} \times [0,1] \to [0,1]$ is a mapping from a transcript, context, and threshold to a real value indicating group-membership. 
    
    If $G(\Pi_t, x, \tau) \in \{0,1\}$ for all $\Pi_t \in \Pi^*, x \in \mathcal{X}, \tau \in [0,1]$, we call $G$ a \emph{binary group}. If $G(\Pi_t, x, \tau_1) = G(\Pi_t, x, \tau_2)$ for all $\Pi_t \in \Pi^*$, $x \in \cX$ and $\tau_1, \tau_2 \in [0,1]$, we call $G$ a \emph{prediction-independent group}.
\end{definition}
We allow group-membership to be real-valued to be able to model scenarios involving partial or probabalistic membership in a group, but in many cases only binary groups need be used. In batch/distributional settings, real valued ``group'' functions can be used to handle distribution shift by encoding likelihood-ratio based reweighting functions --- but this is not needed when we already model the environment as arbitrary/adversarial. Note that the value of a prediction-independent group cannot depend on the prediction being made on that day  - this is relevant in Section \ref{sec:ftrl}, where our results for group conditional coverage hold only for such groups.   
\begin{definition}[Group conditional Coverage, Group Size]
    Given a transcript $\Pi_T$ and a set of groups $\mathcal{G}$, the \emph{coverage} of group $G \in \mathcal{G}$ over $\Pi_T$ is defined as: 
    \begin{equation*}
        \texttt{Cov}(\Pi_T, G) = \frac{1}{T_{G}} \sum_{t=1}^T \textbf{1}[\hat{\tau_t} \geq \tau_t] \cdot G(\Pi_{t}, x_t, \hat{\tau}_t)
    \end{equation*}
     where we define the \textit{size} of the group $T_G = \sum_{t=1}^T G(\Pi_t, x_t, \hat{\tau}_t)$. For a desired coverage rate $q \in (0,1)$, we have \textit{group conditional coverage error} $\gamma$ with respect to $q$ if $|\texttt{Cov}(\Pi_T, G) - q| \leq \gamma$ for all $G \in \mathcal{G}$. 
\end{definition}
Since our setting reduces the problem of building prediction sets to one of predicting a sequence of real-valued parameters, we may ask, in addition to achieving a coverage guarantee, that the sequence of predictions satisfies coverage over groups defined by the level sets of the predicted threshold value, which we call threshold-calibrated coverage. The analogue of group conditional coverage when combined with threshold-calibration is multivalidity. 
\begin{definition}[Threshold-calibrated coverage]
Given a transcript $\Pi_T$ and a desired coverage rate $q \in (0,1)$, we have \textit{threshold-calibrated coverage} with coverage error $\gamma$, if we have group conditional coverage error $\gamma$ with respect to the collection of groups $\mathcal{G} = \{G_\tau: \forall \tau \in [0,1]\}$, where $G_\tau$ is a binary group including all time-steps $t$ for which $\hat{\tau}_t = \tau$. 
\end{definition}

\begin{definition}[Multivalid coverage]
Given a transcript $\Pi_T$, a set of groups $\mathcal{G}$, and a desired coverage rate $q \in (0,1)$, we have \textit{threshold-calibrated coverage} with coverage error $\gamma$, if we have group conditional coverage error $\gamma$ with respect to the new collection of groups $\mathcal{H} = \{H_{G, \tau}: \forall \tau \in [0,1], G \in \mathcal{G}\}$, where $H_{G, \tau}(\Pi_t, x_t, \hat{\tau}_t) = G(\Pi_t, x_t, \hat{\tau}_t)\cdot 1[\hat{\tau}_t = \tau]$ for all $G \in \mathcal{G}, \tau \in [0,1]$.
\end{definition}

\subsection{Regret}
\begin{definition}[$\Phi$-regret \citep{phiregret}]
Given a transcript $\Pi_T = \{(x_t, \tau_t, \hat{\tau}_t)\}_{t=1}^T$, an allowable action space of predictions $\mathcal{A}$, and a loss function $l: \mathcal{A} \times \mathcal{A} \to \mathbb{R}$, the \emph{regret} with respect to the loss function $l$, with respect to a strategy modification rule $\phi: \mathcal{A} \to \mathcal{A}$ is:
\begin{equation*}
    r(\Pi_T, l, \phi) = \sum_{t=1}^T l(\phi(\hat{\tau_t}), \tau_t) - l(\hat{\tau}_t, \tau_t)
\end{equation*}
For any collection of strategy modification rules $\Phi$, we say that $\Pi_T$ has $\Phi$-regret $\gamma$ with respect to $l$ if $r(\Pi_T, l, \phi) \leq \gamma$ for all $\phi \in \Phi$. 
\end{definition}
A transcript has \textit{external regret} if it has $\Phi$-regret with respect to the set of all constant strategy modification rules (of the form $\phi(x) = y$ for all $x \in \mathbb{R}$), and it has \textit{swap regret} if it has $\Phi$-regret with respect to the set of all strategy modification rules. Existing swap-regret algorithms such as  \cite{blum2007external} achieve regret guarantees that have a dependence on the size of the action set $\mathcal{A}$. Therefore, though the parameter values $\{\tau_t\}_{t=1}^T$ are allowed to take any value in $[0,1]$, we will consider a discretized prediction space parametrized by parameter $n$, i.e. predicted values $\hat{\tau}$ can take values only in the set $\mathcal{A}_n = \{0, 1/n, 2/n, \cdots, 1\}$, and the set of strategy modification rules being compared to is $\Phi_n$, the collection of all strategy modification rules $\phi: \mathcal{A}_n \to \mathcal{A}_n$. We can similarly define group conditional external and swap regret given some collection of groups $\mathcal{G}$:
\begin{definition}[$\Phi$-group conditional regret]
Given a transcript $\Pi_T$, an allowable action space of predictions $\mathcal{A}$, a loss function $l: \mathcal{A} \times \mathcal{A} \to \mathbb{R}$, and a set of groups $\mathcal{G}$, the \emph{regret} with respect to the loss function $l$ and group $G$, with respect to a strategy modification rule $\phi: \mathcal{A} \to \mathcal{A}$ is:
\begin{equation*}
    r(\Pi_T, l, \phi, G) = \sum_{t=1}^T \left(l(\phi(\hat{\tau_t}), \tau_t) - l(\hat{\tau}_t, \tau_t)\right)\cdot G(\Phi_t, x_t, \hat{\tau}_t)
\end{equation*}
For any collection of strategy modification rules $\Phi$, we say that $\Pi_T$ has $\Phi$-group conditional regret $\gamma$ with respect to $l$ if $r(\Pi_T, l, \phi, G) \leq \gamma$ for all $\phi \in \Phi$ and $G \in \mathcal{G}$.
\end{definition}
\textit{Group conditional external regret} corresponds to $\Phi$-regret with respect to the set of all constant strategy modification rules, and \textit{group conditional swap regret} corresponds to $\Phi$-regret with respect to the set of all strategy modification rules. There exist efficient algorithms for obtaining diminishing group conditional external and swap regret for any polynomial action space and collection of groups  $\cG$ \citep{blum2020advancing,lee2022online,acharyaoracle,deng2024group}.

To move between no regret and coverage guarantees, note that it is necessary for the threshold parameters to not be too closely clustered together. Suppose for example we had a setting where the empirical distribution defined by $\{\tau_t\}_{t=1}^T$ put all probability mass on a single value $a$. Then the fixed prediction in $\mathcal{A}_n$ closest to $a$ (over all rounds) would achieve no swap-regret, but would correspond to coverage over either all rounds or no rounds, thus being bounded away from any desired coverage rate $q$ for any $q \in (0,1)$. To avoid this general issue, we introduce a smoothness condition that guarantees the parameters we are trying to predict are sufficiently distributed across the support of our probability space. Similar smoothness conditions are common in the online conformal prediction literature \citep{gupta2022online,bastanipractical,gibbs2022conformal}.

\begin{definition}[$(\alpha, \rho,r)$-smoothness] 
    A distribution $\mathcal{D} \in \Delta [0,1]$ is said to be $(\alpha, \rho, r)$-smooth if for every pair of values $p,q$ such that $0 \leq p \leq q \leq 1$ and $|p-q| \leq 1/r$, we have $\mathbb{P}_{\tau \sim \mathcal{D}}[\tau \in [p,q]] \leq \rho$, and if $|p-q| >= 1/r$, then $\mathbb{P}_{\tau \sim \mathcal{D}}[\tau \in [p,q]] \geq \alpha$. 
\end{definition}

% \begin{definition}[$(\alpha, \rho)$-smoothness]
%     A distribution $\mathcal{D} \in \Delta[0,1]$ with density $p(\tau)$ is said to be $(\alpha, \rho)$-smooth if
%     \begin{equation*}
%         0 < \alpha \leq p(\tau) \leq \rho ~~~~ \forall \tau \in [0,1]
%     \end{equation*}
%     If $\mathcal{D}$ is a discrete distribution with probability mass function $p: S_{\mathcal{D}} \to [0,1]$ where $S_{\mathcal{D}}$ is the support of $\mathcal{D}$, then it similarly satisfies the smoothness condition if 
%     \begin{equation*}
%         0 < \alpha \leq p(\tau) \leq \rho ~~~~ \forall \tau \in S_{\mathcal{D}}. 
%     \end{equation*}
% \end{definition}

\section{Coverage Guarantees Through Regret}
\label{sec:swap-regret}

In stochastic settings without context, external regret (under some mild smoothness conditions) is sufficient to obtain marginal coverage. To prove this, we first draw a connection between the expected difference in pinball loss between two thresholds and their absolute difference, using a slightly modified version of Proposition 5 from \cite{gibbs2022conformal}:
\begin{restatable}{lemma}{iidloss}
    \label{lem:iid}
    Fix a distribution $\mathcal{D}$, and let $\tau^*$ be the $q$-th quantile of $\mathcal{D}$. Then, assuming $\mathcal{D}$ is an $(\alpha,\rho, r)$-smooth distribution, for any other threshold $\tau'$,  
    \begin{equation*}
        \frac{\alpha r \cdot (\tau^* - \tau')^2}{2} \leq \E_{\tau \sim \mathcal{D}}[p_q(\tau', \tau)] - \E_{\tau \sim \mathcal{D}}[p_q(\tau^*, \tau)] 
    \end{equation*}
\end{restatable}

With this, we can move from a regret bound to a bound on our miscoverage rate. We give a sketch of the proof here, with the full version in the appendix. 

\begin{restatable}{theorem}{stochcov}
    \label{thm:stoch-cov}
    Fix a transcript $\Pi_T = \{(\tau_t, \hat{\tau}_t)\}_{t=1}^T$ in a setting without context (i.e. in which there are no observable features $x_t$) and where the sequence of labels is drawn IID from a fixed distribution, i.e $\tau_t \sim \mathcal{D}$ for all $t \in [T]$. If $\mathcal{D}$ is $(\rho, r)$ smooth, and if $\Pi_T$ has external regret $\gamma$ with respect to the negative of pinball loss $-p_q$. then the set of predicted thresholds satisfies marginal coverage at the level $q$:
    \begin{equation*}
        |\texttt{Cov}(\Pi_T) - q| \leq \sqrt{\frac{2 \rho (\gamma + \epsilon)}{T \alpha}} + \frac{\epsilon}{T}
    \end{equation*}
    with probability at least $1 - 4\exp\left(-\frac{\epsilon^2}{2T}\right)$.
\end{restatable}
\begin{proof}[Proof Sketch] We are given an upper-bound on the realized regret with respect to $-p_q$, which with high probability is close to the expected regret (using Azuma's inequality). We then bound the sum of squared differences between the optimal threshold $\tau^*$ (which minimizes simultaneously the expected pinball loss and deviation of expected coverage from $q$) and the predicted thresholds $\hat{\tau}_t$ using Lemma \ref{lem:iid}. The smoothness condition implies thresholds that are close together must have similar expected coverages, and another application of Azuma's inequality proves that this is (with high probability) close to the realized coverage. 
\end{proof}

In the non-contextual setting, when we have sublinear external regret, the bound above goes to zero as $T$ increases. But in adversarial settings, there is no such connection between external regret and coverage even in the non-contextual setting.
\begin{example}
    \label{ex:1}
    Define the transcript $\Pi_T = \{(\tau_t, \hat{\tau}_t)\}_{t=1}^T$ in the non-contextual setting where the predicted threshold $\hat{\tau}_t = 0.4$ for odd $t$ and $0.9$ for even $t$, and each day $y_t$ is chosen by the adversary such that $\tau_t = 0.5$ for odd $t$ and $1$ for even $t$. On each day $t$, the loss with respect to $-p_q$ (for $q = 0.5$) is $-0.1$, and since the true thresholds distribute evenly over the set $\{0.5, 1\}$, the best fixed threshold $\tau^*$ in hindsight is the median $0.75$ which achieves a loss of $-0.25$ every day. Therefore $\sum_{t=1}^T p_q(\hat{\tau}_t, \tau_t) - p_q(\tau^*, \tau_t) \leq 0$, i.e this transcript has no regret with respect to pinball loss at the level $q=0.5$. However, the predicted threshold is always lower than the true threshold, and so $\texttt{Cov}(\Pi_T) = 0$. 
\end{example} 
In fact, the connection between low regret on a sequence and achieving low miscoverage on that same sequence falls apart even in i.i.d. settings when we move to group conditional coverage. Theorem \ref{thm:stoch-cov} is driven by the fact that in non-contextual settings, the prediction made each day is independent of the realized outcome drawn from $\mathcal{D}$. When we introduce groups, we move to the contextual setting. % in which the prediction made by the algorithm must be allowed to be a function of the group membership of the example at each round.
As soon as we allow the threshold to depend on context, we find that external regret no longer implies coverage even in i.i.d. settings because of the correlation that the context introduces between our predictions and the outcomes.
\begin{example}
    \label{ex:2}
    Define the context space $\mathcal{X} = \{A, B\}$, and suppose we are interested only in marginal coverage over the group $G$ containing all days. The distribution $\mathcal{D}$ over $(x,y)$ pairs is defined such that we randomize uniformly over contexts (i.e. $\p(x = A) = \p(x = B) = 0.5)$, and  non-conformity score function $f$ is such that $f(A, .) = 0.5$, and $f(B, .) = 1$. Then an algorithm $\mathcal{A}$ that always predicts a threshold $\hat{\tau}_t = 0.4$ when $x_t = A$, and a threshold $\hat{\tau}_t = 0.9$ when $x_t = B$, simulates the environment described in Example \ref{ex:1}. Thus $\mathcal{A}$ will have negative expected regret (which will be arbitrarily close to the realized regret with high probability for sufficiently large $T$), but always a realized coverage of zero. 
\end{example}

To make further connections between regret and coverage, we will need to move to stronger guarantees. Informally, the reason why external regret is not sufficient to give coverage guarantees is that the benchmark class that regret guarantees compare to corresponds to \emph{fixed} thresholds, whereas the algorithm may vary its thresholds over time, and if these are correlated with the outcome, it might result in lower pinball loss than any fixed threshold despite the fact that it never covers the label. Swap regret will fix this issue, informally, because it requires that regret be low not just marginally, but on subsequences in which the algorithm's threshold is fixed, thus putting the algorithm and the benchmark on equal footing. We will make use of a discretized version of Lemma \ref{lem:iid}:

\begin{restatable}{lemma}{discreteineq}
    \label{lem:discrete-ineq}
    Given a set of parameter values $\{\tau_i\}_{i=1}^T$, and any two fixed values $a,b \in [0,1]$, define the sum of pinball losses $L_a = \sum_{i=1}^T p_q(a, \tau_i)$ and $L_b =  \sum_{i=1}^T p_q(b, \tau_i)$ respectively, where $a = \min_{\tau \in \mathcal{A}_n} \sum_{i=1}^T p_q(a, \tau_i)$ is the minimizer of the sum of pinball losses over $\mathcal{A}_n$. If the empirical distribution $\mathcal{D}$ defined by $\{\tau_i\}_{i=1}^{T}$ is ($\alpha, \rho, r)$-smooth, and if $L_b - L_a \leq \gamma$, then $|a-b| \leq \sqrt{\frac{2\gamma}{T \alpha r}}$.
\end{restatable}

\begin{restatable}{theorem}{regrettocov}
\label{thm:regret-to-cov}
   Fix a transcript $\Pi_T = \{(x_t, \tau_t, \hat{\tau}_t)\}_{t=1}^T$. If $\Pi_T$ has swap regret $\gamma$ with respect to the negative of pinball loss $-p_{q}$, and the empirical distribution $\mathcal{D}_\tau$ defined by the set $\{\tau_t\}_{t: \hat{\tau} = \tau}$ is $(\alpha, \rho,r)$-smooth for each $\tau \in \mathcal{A}_n$, then the set of predicted thresholds satisfies threshold-calibrated coverage at the level $q$:
\begin{equation*}
    | \texttt{Cov}(\Pi_T, G_\tau) - q| \leq \frac{\rho}{2} + \frac{\rho r}{n} +\sqrt{\frac{2\gamma}{T_{G, \tau} \alpha r}}
\end{equation*}
\end{restatable}
\begin{proof}[Proof Sketch] The swap regret guarantee gives an upper-bound on the regret of the subsequence defined by all time-steps making a fixed prediction $\tau$. Due to convexity of the pinball loss function, the true minimizer $M(\tau)$ of the sum of pinball losses must be close to the minimizer in the discrete set $\mathcal{A}_n$, which in turn can be bound closely to $\tau$ using Lemma $\ref{lem:discrete-ineq}$ and the regret bound. The smoothness condition implies not a lot of probability weight can be placed in the interval $|M(\tau) - \tau|$, and so the difference in coverage is also small. Since $M(\tau)$ should achieve the desired coverage rate $q$, this gives a bound on the miscoverage on the subsequence defined by any fixed prediction $\tau$ (for all $\tau \in \mathcal{A}_n$). 
\end{proof}
The full proof can be found in the appendix. Note that if $\gamma$ (as a function of $T$) grows sublinearly, then the final term in the above inequality goes to zero as $T$ becomes arbitrarily large. Several existing swap-regret algorithms \cite{blum2007external} achieve such rates. We can also move from threshold-calibrated coverage to regret bounds. The proof is similar in idea to Theorem \ref{thm:cov-to-regret}, so we relegate it to the appendix. 

\begin{restatable}{theorem}{covtoregret}
    \label{thm:cov-to-regret}
    Fix a transcript $\Pi_T$. If $\Pi_T$ has threshold-calibrated coverage with coverage error $\gamma$ (at desired coverage rate $q$), and $\mathcal{D}_\tau$ is $(\alpha, \rho, r)$-smooth for each $\tau \in \mathcal{A}_n$, then the transcript also has swap regret with respect to the loss $-p_q$ bounded by:
    \begin{equation*}
        r(\Pi_t, -p_q, \phi) \leq \frac{T \gamma^2 \rho}{\alpha^2 r}
    \end{equation*}
    for each $\phi \in \Phi$, the collection of all strategy modification rules for action set $\mathcal{A}_n$.
\end{restatable}
and so if $\frac{1}{\gamma^2}$ grows at a rate faster than $T$, we achieve sublinear regret. Applying the same analysis in the context of group conditional swap regret (analyzing subsequences determined by a fixed predicted threshold and group inclusion) gives us an analagous relationship between group conditional swap regret and multivalid guarantees. 
\begin{restatable}{theorem}{gp1}
    \label{thm:gp-regret-to-cov}
    Fix a transcript $\Pi_T$, and a set of binary groups $\mathcal{G}$. If $\Pi_T$ has group conditional swap regret $\gamma$ with respect to the negative of pinball loss $-p_{q}$, and the empirical distributions $\mathcal{D}_{G, \tau}$ defined by the set $\{\tau_t\}_{t: \hat{\tau} = \tau, t \in G}$ are $(\alpha,\rho,r)$-smooth, then the set of predicted thresholds satisfies multivalid coverage at the level $q$:
    \begin{equation*}
        | \texttt{Cov}(\Pi_T, H_{G, \tau}) - q| \leq \frac{\rho}{2} + \frac{\rho r}{n} +\sqrt{\frac{2\gamma}{T_{G, \tau}\alpha r}}
    \end{equation*}
    for each group $H_{G, \tau}$, defined as $H_{G,\tau}(\Pi_t, x_t, \hat{\tau}_t)) = G(\Pi_t, x_t, \hat{\tau_t}) \cdot \one[\hat{\tau}_t = \tau]$.
    \end{restatable}

\begin{restatable}{theorem}{gp2}
    \label{thm:gp-cov-to-regret}
    Fix a transcript $\Pi_T$, and a set of binary groups $\mathcal{G}$. If $\Pi_T$ has multivalid coverage with coverage error $\gamma$ (at desired coverage rate $q$), and $\mathcal{D}_{G, \tau}$ is $(\alpha,\rho,r)$-smooth for each $\tau \in \mathcal{A}_n, G \in \mathcal{G}$, then the transcript also has group conditional swap regret with respect to the loss $-p_q$, such that:
    \begin{equation*}
            r(\Pi_t, -p_q, \phi, G) \leq \frac{T \gamma^2 \rho}{\alpha^2 r}
    \end{equation*} 
    for each $G \in \cG$ and $\phi \in \Phi$, the collection of all strategy modification rules for action set $\mathcal{A}_n$.
\end{restatable}

\label{sec:ftrl-cov}
\begin{algorithm}
    \caption{Follow The Regularized Leader (Group Conditional Coverage)}
    \label{alg:ftrl}
\begin{algorithmic}
 \STATE {\bfseries Input:} Timesteps $T$,  regularizer $R: [0,1] \to \mathbb{R}$, loss parameter $q$
\FOR{$t = 1, 2, \cdots T$}
    \STATE Receive $g_t$ from adversary
    \STATE Choose $\theta_t = \arg\min_{\theta \in \mathbb{R}^d} \sum_{s=1}^{t-1} l_t(\theta, \tau_s)+ R(\theta)$.
    \STATE Predict $\hat{\tau}_t = \langle \theta_t, g_t \rangle$
    \STATE Receive $\tau_t$ from the adversary. 
    \STATE Define loss $l_t(\theta, \tau_t) = \langle \theta, \nabla_{\theta} p_q(\langle \theta_t, g_t \rangle, \tau_t) \rangle$
\ENDFOR
\end{algorithmic}
\end{algorithm}

\section{Coverage Guarantees for FTRL Algorithms}
\label{sec:ftrl}
Having established that, in general, external regret guarantees with respect to the pinball loss do not imply non-trivial coverage on their own, either in adversarial settings, or in settings with context (as relevant to group conditional coverage) even in the i.i.d. setting, we turn our attention to a particular (but broad) class of no regret learning algorithms --- those in the ``follow the regularized leader'' (FTRL) family. This class of algorithms includes multiplicative weights, online gradient descent, and many other algorithms. At a high level, an algorithm in the FTRL family receives a loss function $\ell(a, y)$ at every iteration, parameterized by a choice of action $y$ by the adversary and a choice of action $a \in \mathbb{R}^d$ by the learner. The loss is assumed to be linear in $a$ for all choices $y$ of the adversary. An instantiation of FTRL is given by a convex \emph{regularization} function $R:\mathbb{R}^d \rightarrow \mathbb{R}$, and the action that FTRL plays at every iteration is $a_t = \arg\min_{a} \sum_{s=1}^{t-1} \ell(a,y_s) + R(a)$ --- the \emph{regularized} empirical risk minimizer on the empirical loss distribution so far. Follow the regularized leader can also be used with loss functions $\hat \ell(a,y)$ that are \emph{convex} in $a$. In this case, the algorithm takes as input the linear loss function $\ell(a,y) \doteq \langle a, \nabla_a \hat\ell(a_t,y) \rangle$ --- defined by the \emph{gradient} of the loss function evaluated at the point $a_t$ the learner plays at round $t$. This reduces to the linear case and obtains the same regret bound (see e.g. \cite{shalev2012online}). Online gradient descent is an instance of FTRL regularized by the Euclidean norm; multiplicative weights is an instance of FTRL regularized by entropy; other algorithms follow from different regularization functions. 

In this section we consider coverage guarantees for   algorithms in the FTRL family when  actions  for the learner are parameter vectors $\theta_t \in \mathbb{R}^k$, actions for the adversary are nonconformity scores $\tau_t \in [0,1]$ and the loss function is $p_{q}(\langle \theta_t, g_t \rangle, \tau_t)$ --- the pinball loss (at some target quantile $q$) of the prediction $\hat \tau_t \doteq \langle \theta_t, g_t \rangle$ with respect to $\tau_t$. Here $g_t \in [0,1]^k$ is the vector of group membership for the example at round $t$, i.e. given $k$ prediction-independent groups $\{G_1, \cdots, G_k\}$, $g_{t,i} = G_i(\Pi_t, x_t)$. We show that in this setting, for all algorithms in the FTRL family, the miscoverage rate can be bounded as a function of the magnitude of the parameter $\theta_t$ and the gradient of the regularization function $R(\cdot)$.

% We first give some notation:
% \begin{definition} \ar{Didn't we already define coverage? Weird to define it a second time with different notation}
%     Fix a transcript $\Pi_T = \{(x_t, y_t, \hat{\tau}_t, g_t)\}_{t=1}^T$, and a non-conformity score $f$, where $g_t \in [0,1]^k$ is the group-indicator vector on day $t$, for real-valued groups. We index the groups by $i \in [k]$, the \emph{coverage} of group $i$ in $\Pi_T$ is defined as:
%     \begin{equation*}
%         \texttt{Cov}(\Pi_T, i) = \frac{1}{T_i} \sum_{t=1}^T \textbf{1}[\hat{\tau_t} \geq \tau_t] \cdot g_{t, i}
%     \end{equation*}
%     where for each round $t$, $\tau_t = f(x_t, y_t)$, and we define $T_i$, the size of group $i$, as:
%     \begin{equation*}
%         T_i = \sum_{t=1}^T g_{t,i}
%     \end{equation*}
% \end{definition}

\begin{restatable}{theorem}{FTRL}
\label{thm:FTRL}
    For the parametrization of FTRL given in Algorithm \ref{alg:ftrl} with regularization function $R:\mathbb{R}^d\rightarrow \mathbb{R}$, for any target coverage rate $q$ and any $T$ the resulting transcript $\Pi_T$ is guaranteed to satisfy group conditional coverage for groups $G_i$ ($i \in [k]$) at the rate:
    $$|\texttt{Cov}(\Pi_T, G_i) - q| \leq \frac{||\nabla R(\theta_{T+1})||_\infty}{T_i}$$
\end{restatable}
\begin{proof}
    Pinball loss is convex, and so to apply FTRL, we feed the algorithm the linear surrogate loss $\ell(\theta,\tau_t) \doteq \langle \theta, \nabla_\theta p_{q}(\langle \theta_t, g_t \rangle, \tau_t) \rangle$. We can compute:
    $$\ell(\theta,\tau_t) = \begin{cases}
-\,q\,\langle \theta,g_t\rangle, 
& \text{if } \tau_t > \langle \theta_t, g_t\rangle,\\[6pt]
(1 - q)\,\langle \theta,g_t\rangle, 
& \text{if } \tau_t \leq \langle \theta_t, g_t\rangle.
\end{cases}$$

The gradient of the loss at round $t$ with respect to $\theta$ is therefore:
    $$\nabla_\theta \ell(\theta, \tau_t) = \begin{cases}
-\,q\,g_t, 
& \text{if } \tau_t > \langle \theta_t, g_t\rangle,\\[6pt]
(1-q)\,g_t, 
& \text{if } \tau_t \leq \langle \theta_t, g_t\rangle.
\end{cases}$$

FTRL with regularizer $R$ plays the action $\theta_t$ at round $t$ that solves:
$$\theta_t = \arg\min_{\theta}  \sum_{s=1}^{t-1} \ell(\theta,\tau_s) + R(\theta)$$
First order optimality conditions imply that:
$$\sum_{s=1}^{t-1}\nabla_\theta \ell(\theta_t, \tau_s) + \nabla R(\theta_t) = 0$$
Or equivalently, 
\begin{eqnarray*}
\nabla R(\theta_t) &=& \sum_{s : \tau_s > \langle \theta_s, g_s\rangle} q g_s + \sum_{s : \tau_s \leq \langle \theta_s, g_s\rangle}(q-1)g_s \\
&=& \sum_{s=1}^{t-1} g_s (q - 1[\tau_s \leq \hat \tau_s]) 
\end{eqnarray*}
Hence we can bound the miscoverage rate for every group $i$ at time $T$ can be bounded as:
\begin{eqnarray*}
    |\texttt{Cov}(\Pi_T, G_i) - q| \leq \frac{||\nabla R(\theta_{T+1})||_\infty}{T_i}
\end{eqnarray*}
\end{proof}

This theorem tells us whenever we can upper-bound $||\nabla R(\theta_{T+1})||_\infty$ by any function that grows sublinearly with $T$, we get a non-trivial group conditional coverage bound. In the following section we do this for online gradient descent, an especially simple instantiation of FTRL.  

\section{Group Conditional ACI}
\label{sec:groupwise-cov}
Algorithms such as ACI (``Adaptive Conformal Inference'') from \cite{gibbs2021adaptive} can be seen as special cases of the connection between FTRL and coverage guarantees we have shown --- in particular the special case in which we ask only for marginal coverage, and use gradient descent with step size $\eta$, which is an instantiation of FTRL in which the regularization function $R(\theta) = \frac{1}{2\eta}||\theta||^2$. We give the ``gradient descent'' implementation of our algorithm in Algorithm \ref{alg:group-coverage}.

\begin{algorithm}
    \caption{Group Conditional ACI (GCACI)}
    \label{alg:group-coverage}
\begin{algorithmic}
 \STATE {\bfseries Input:} Timesteps $T$, number of groups $k$, coverage target $q$, step-size $\eta$
\STATE Choose $\theta_1 = \textbf{0}$. 
\FOR{$t = 1, 2, \cdots T$}
    \STATE Receive $\textbf{g}_t$ from the adversary.
    \STATE Predict $\hat{\tau}_t = \langle \theta_t, \textbf{g}_t \rangle$.
    \STATE Receive $\tau_t$ from adversary. 
    \IF{$\langle \theta_t, \textbf{g}_t \rangle < \tau_t $} 
        \STATE $\theta_{t+1} = \theta_t + \eta \cdot q \cdot \textbf{g}_t$ ~~~~~~~~~~~~~~~~\text{(Update A)}
    \ELSE
        \STATE $\theta_{t+1} = \theta_t - \eta \cdot (1-q) \cdot \textbf{g}_t$ ~~~~~~~~~\text{(Update B)}
    \ENDIF
\ENDFOR
\end{algorithmic}
\end{algorithm}

We can instantiate our Theorem \ref{thm:FTRL} to bound the group conditional miscoverage of Algorithm \ref{alg:group-coverage}, as it is a special case of FTRL. 

\begin{restatable}{lemma}{groupcov}
    \label{lem:group-cov}
    Running Algorithm \ref{alg:group-coverage} for any number of rounds $T$ with a coverage target of $q$ for any set of $k$ group functions, we achieve group conditional miscoverage bounded by the following function of  $\theta_{T+1}$:
    \begin{equation*}
        \left|\texttt{Cov}(\Pi_T, G_i)-q\right| \leq \frac{||\theta_{T+1}||_\infty}{T_i \eta}
    \end{equation*}
\end{restatable}
\begin{proof}
    Algorithm \ref{alg:group-coverage} is an instantiation of follow the regularized leader as analyzed in Theorem \ref{thm:FTRL} with regularization function $R(\theta) = \frac{1}{2\eta}||\theta||^2$. We can compute $\nabla R(\theta_{T+1}) = \frac{1}{\eta} \cdot \theta_{T+1}$. Plugging this into Theorem \ref{thm:FTRL} gives the stated bound.
\end{proof}

\iffalse
\begin{proof}
    By definition of the update steps, note that:
    \begin{equation*}
        \theta_{t+1, i} = \theta_{t, i} + \eta \cdot (\one[\hat{\tau}_t < \tau_t] - (1-q)) \cdot g_{t, i}
    \end{equation*}
    for each $i \in [k]$ and $t \in [T]$. Thus, combining all $T$ updates, we have:
    \begin{align*}
        \theta_{T+1, i} &= \theta_{1, i} + \eta \sum_{t=1}^T (\one[\hat{\tau}_t < \tau_t] - (1-q)) \cdot g_{t, i} \\
        &= \eta \sum_{t=1}^T \one[\hat{\tau}_t < \tau_t] \cdot g_{t,i} - \eta \alpha \sum_{t=1}^T g_{t,i} \\
        &= \eta T_i((1 - \texttt{Cov}(\Pi_T, i)) - \alpha )
    \end{align*}
    Rearranging gives us the desired equality. 
\end{proof}
\fi
If we are able to upper-bound the magnitude of the last iterate of gradient descent as a sublinear function of $T$, we can bound the deviation from desired coverage not just marginally, but groupwise for arbitrary intersecting groups:

\begin{restatable}{lemma}{iterateval}
    When Algorithm \ref{alg:group-coverage} is run with step-size $\eta \in (0,1]$, for any collection of $k$ group functions, any coverage target $q \in (0,1)$, and every $T$,  the iterate $\theta_{T+1}$ has norm bounded as:
    $$||\theta_{T+1}||_\infty \leq \mathcal{O}(\sqrt{\eta T (\eta k + 1)})$$    
\end{restatable}
\begin{proof}
First, note that since the non-conformity scores are assumed to be bounded in $[0,1]$, we must have for every $t$ that $\tau_t \in [0,1]$. So, if $\langle \theta_t, g_t \rangle < 0$, we must also have $\langle \theta_t, g_t \rangle \leq \tau_t$ which triggers Update A. Similarly, whenever $\langle \theta_t, g_t \rangle \geq 1$, this necessarily triggers update B. Said another way: if the update A was triggered at round $t$ we know that   $\langle \theta_t, g_t \rangle < 1$, whereas if update B was triggered, we know that $\langle \theta_t, g_t \rangle \geq 0$. We consider these two cases separately.

\textbf{Case 1 (Update A triggered):} We can compute
\begin{align*}
    \lVert \theta_{t+1} \rVert_2^2 &= \lVert \theta_{t} \rVert_2^2 + \eta^2q^2 \lVert g_t \rVert_2^2 + 2\eta q \langle \theta_t, g_t \rangle \\ 
    &\leq \lVert \theta_t \rVert_2^2 + \eta^2q^2 k + 2\eta q
\end{align*}

\textbf{Case 2 (Update B triggered):} Similarly,
\begin{align*}
    \lVert \theta_{t+1} \rVert_2^2 &= \lVert \theta_{t} \rVert_2^2 + \eta^2(1-q)^2 \lVert g_t \rVert_2^2 - 2\eta(1-q) \langle \theta_t, g_t \rangle \\ 
    &\leq \lVert \theta_t \rVert_2^2 + \eta^2(1-q)^2 k
\end{align*}
As initially  $\lVert \theta_1 \rVert_2 = 0$, we obtain:
\begin{align*}
   \lVert \theta_{T+1} \rVert_2^2 &\leq T \left(\eta^2 k \max\{q, 1-q\}^2 + 2 \eta q\right) \\
   &\leq T \eta\left(\eta k \max\{q, 1-q\}^2 + 2 q\right)
\end{align*}
This immediately gives us a bound on the $L_{\infty}$ norm:
\begin{equation*}
    ||\theta_{T+1}||_{\infty} \leq \sqrt{T\eta}\sqrt{\eta k \max\{q, 1-q\}^2 + 2q}
\end{equation*}
\end{proof}
% \begin{proof}[Proof Sketch] We can write $\lVert \theta_{t+1} \rVert_2^2$ in terms of $\lVert \theta_{t} \rVert_2, \lVert g_{t} \rVert_2,$, and $\langle \theta_t, g_t \rangle$ using an inner-product expansion. By definition of the update rules, Update A and Update B being triggered imply an upper and lower bound on $\langle \theta_t, g_t \rangle$ respectively, and recursively applying this bound over all $T$ rounds gives us a bound on  $\lVert \theta_{t+1} \rVert_2^2$ (and thus on $\lVert \theta_{t+1} \rVert_{\infty}$) as desired. 
% \end{proof}

Putting these two lemmas together gives us a group conditional coverage bound for Algorithm \ref{alg:group-coverage} (GCACI):
\begin{theorem}Fix any collection of $k$ group functions taking values in $[0,1]$ and any target coverage rate $q \in (0,1)$. If we run Algorithm \ref{alg:group-coverage} for $T$ rounds with step size $\eta \in (0,1]$, we achieve group conditional miscoverage bounded by:

       $$ \left|\texttt{Cov}(\Pi_T, G_i)-q\right| \leq O\left(\frac{\sqrt{\eta T(\eta k+1)}}{T_i \eta}\right)$$

\end{theorem}

When we set $\eta = 1$, this gives us a $O(\sqrt{T k}/T_i)$ group conditional coverage error bound. This analysis is tight even for $k = 1$ if we allow the groups to be real valued.
\begin{restatable}{theorem}{counterexample}
    Let $k = \eta = 1$ and pick any coverage target $q \in (0,1)$. The sequence of 1-dimensional weighting functions $g_t = \frac{1}{2\sqrt{t-1}}$ together with thresholds $\tau_t = 1$ causes Algorithm \ref{alg:group-coverage} to produce parameter vector $\theta_{T+1} \in \Omega(\sqrt{T})$. 
\end{restatable}

We remark that this lower bound construction seems to require real valued group functions. We conjecture that a much better upper bound on $||\theta_{T+1}||_\infty$ is true for \emph{binary} valued group functions --- growing much more slowly with (or perhaps even independently of) $T$. Our experiments support this conjecture, but we are unable to prove it.

\section{Experiments}
 \label{sec:experiments}

In this section we compare the performance of Algorithm \ref{alg:group-coverage} with that of the MVP (``multi valid predictor'') algorithm (\cite{bastanipractical}), that to our knowledge is the only other  method for obtaining non-trivial group-conditional coverage guarantees in sequential adversarial settings. We run experiments on the same collection of datasets used to evaluate MVP in \cite{bastanipractical}. We compare rates of convergence to the desired coverage over all groups. Since the guarantees for our algorithm are more fine grained, and are proven in terms of $||\theta_t||_\infty$, we plot also the $L_{\infty}$ norm of the parameters $\theta_t$ maintained by Algorithm \ref{alg:group-coverage} over time. To achieve our derived $O(\sqrt{T k}/T_i)$ bounds we set the learning rate $\eta = 1$ for these experiments. We also then empirically investigate the relationship between the rate of convergence to the target coverage rate and the learning rate, by measuring the time-step\footnote{Here, time-step is defined as within the subsequence defined by a group, not the full sequence.} at which the empirical group conditional coverage for the rest of the sequence falls within $\epsilon$ of the desired coverage rate, as a function of $\eta$. We set $\epsilon = 0.01$ for all tests.

\paragraph{Time Series Data}  We replicate the prediction task described first in \cite{gibbs2021adaptive}, for testing the ACI algorithm's ability to achieve marginal coverage, which uses AMD stock market data from the WSJ daily price across years 2000-2020. The dataset gives price points $\{p_t\}_{t=1}^T$ of the stock for $T = 5283$. Using this data, we compute the daily return $r_t$, defined as $r_t = \frac{p_t - p_{t-1}}{p_{t-1}}$, which correspondingly defines the daily realized volatility $v_t = r_t^2$. The task is to predict this volatility. Using the predictive model GARCH (\cite{garch_paper}), which makes a prediction of the volatility $\hat{v}_t$, the non-conformity score used on day $t$ is $f_t(x, y) = \frac{|y - \hat{v}_t|}{\hat{v}_t}$, normalized to ensure scores are always in the range $[0,1]$\footnote{Note that though the feature vector $x$ is included generically here as an argument in the non-conformity score, the GARCH model typically uses only past volatility data to make predictions for the next time-step in the sequence.}. Then, as in \cite{bastanipractical}, we define the collection of 20 groups $\{G_i\}_{i=1}^{20}$, where $G_i$ includes all time-steps $t$ for which $t \equiv 0$ (mod $i$), and introduce artificial noise to group data in the following way - for each time-step $t$, we add noise $\mathcal{N}(0, \hat{\sigma}_r)$ to the return value $t$ for each group in $G_i$ that $t$ is included in, where $\hat{\sigma}_r$ is the standard deviation of the original return sequence. We run both GCACI and MVP on this data, asking for a desired group conditional coverage level of $q = 0.9$. In \cite{bastanipractical}, they show that MVP achieves the desired group conditional coverage while ACI is unable to. Here, we see that GCACI not only achieves group conditional coverage, but converges at much quicker rates than MVP. 

\begin{figure*}[t!]
\begin{minipage}[b]{0.45\linewidth}
    \centering
    \includegraphics[width=\linewidth]{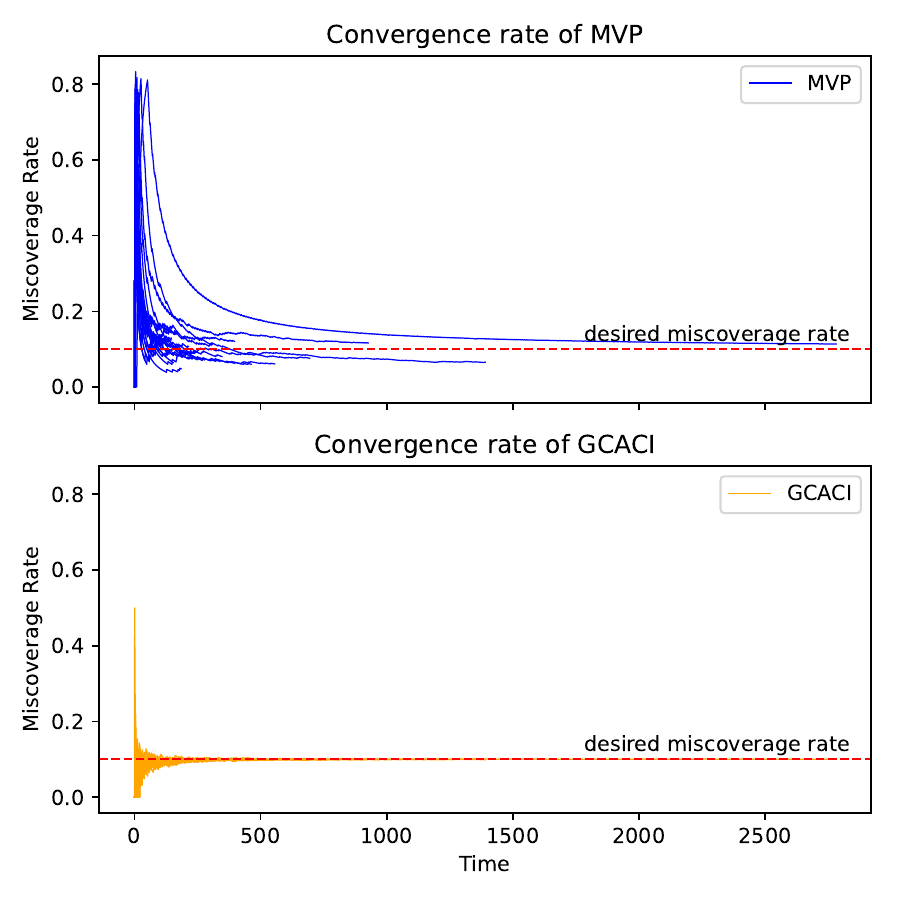}
    % \caption{Time Series Data}
\end{minipage}
\begin{minipage}[b]{0.45\linewidth}
    \centering
    \includegraphics[width=\linewidth]{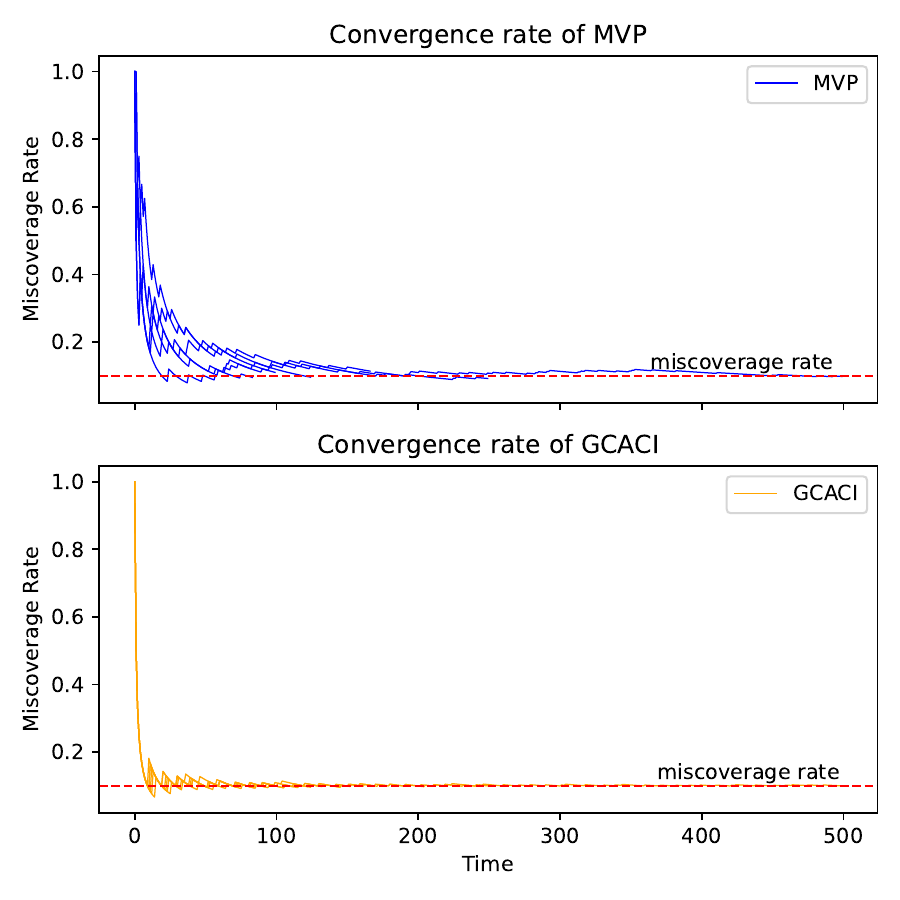}
    % \caption{UCI Data}
\end{minipage}
\centering
\begin{minipage}[b]{0.45\linewidth}
    \centering
    \includegraphics[width=\linewidth]{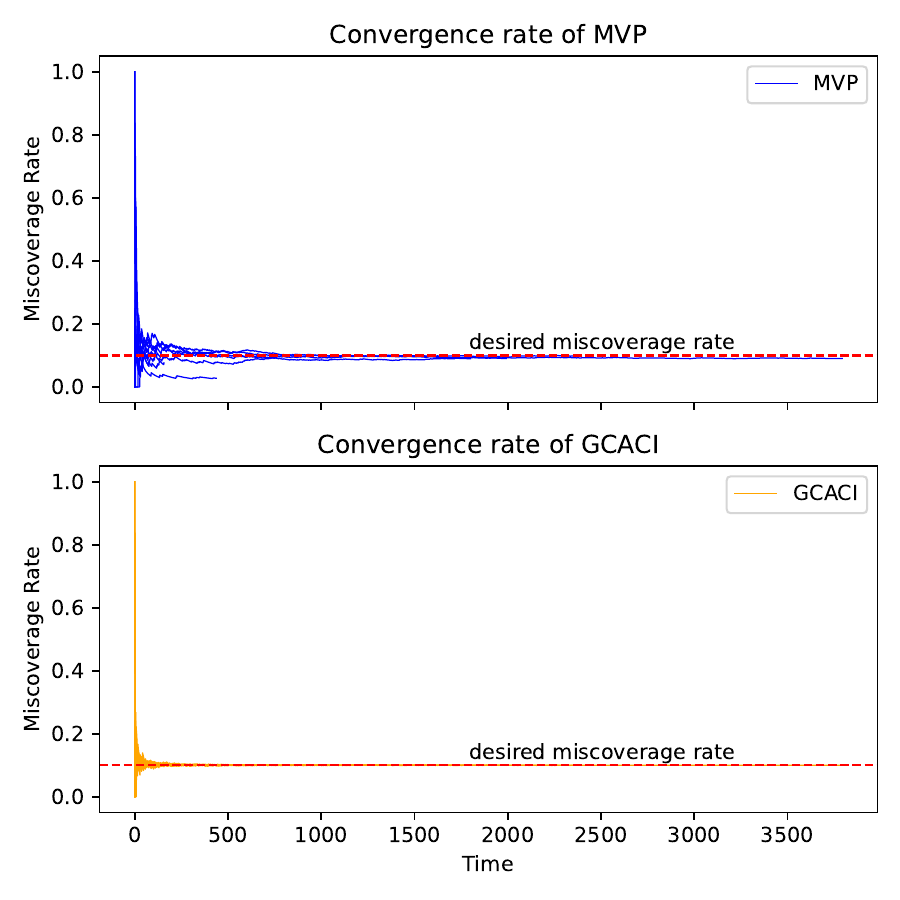}
        % \caption{Folktables Data}
\end{minipage}
\caption{Comparison of convergence rates between GroupACI (GCACI) and MVP for group coverage. Each curve captures the averaged miscoverage over time of a single group. The top left is the comparison for time series data, the top right is UCI Airfoil data, and the bottom is Folktables data. Note that group size varies for curves in each graph.}
\label{fig:convergence}
\end{figure*}

\paragraph{Synthetic Distribution Shift (UCI Airfoil Data)} We run both MVP and GCACI on the airfoil dataset from the UCI Machine Learning Repository (\cite{Dua:2019}), which consists of $1503$ instances of NASA airfoil blades; the task is to predict the Scaled Sound Pressure Level (SSPL). In \cite{bastanipractical}, they compare against the performance of the weighted split conformal prediction algorithm of \cite{tibshirani2019conformal}, and test only for marginal coverage. Following their approach, we use 25\% of the data to train a linear regression model $g: \cX \to \mathbb{R}$, which defines the non-conformity score $f(x, y) = |g(x) - y|$. Another 25\% of the data is used as is, and the final 50\% of the data is sampled (with replacement)  using exponential tilting - each datapoint $x$ is drawn with probability proportional to $\exp(\langle x, \beta \rangle)$, where we set $\beta = (-1, 0, 0, 0, 1)$ as in \cite{tibshirani2019conformal} and \cite{bastanipractical}, representing synthetic covariate shift. The test set is sequenced such that the original (unshifted) data comes first, followed by the shifted data. Then, as in the previous section, we define a set of six groups $\{G_i\}_{i=1}^{6}$ where membership is again defined by time-step, i.e. $G_i$ includes all time-steps for which $t \equiv 0$ (mod $i$). Both algorithms are run with a desired coverage rate $q = 0.9$.

\begin{figure}[t!]
    \centering
    \includegraphics[width=0.7\linewidth, height = 9cm]{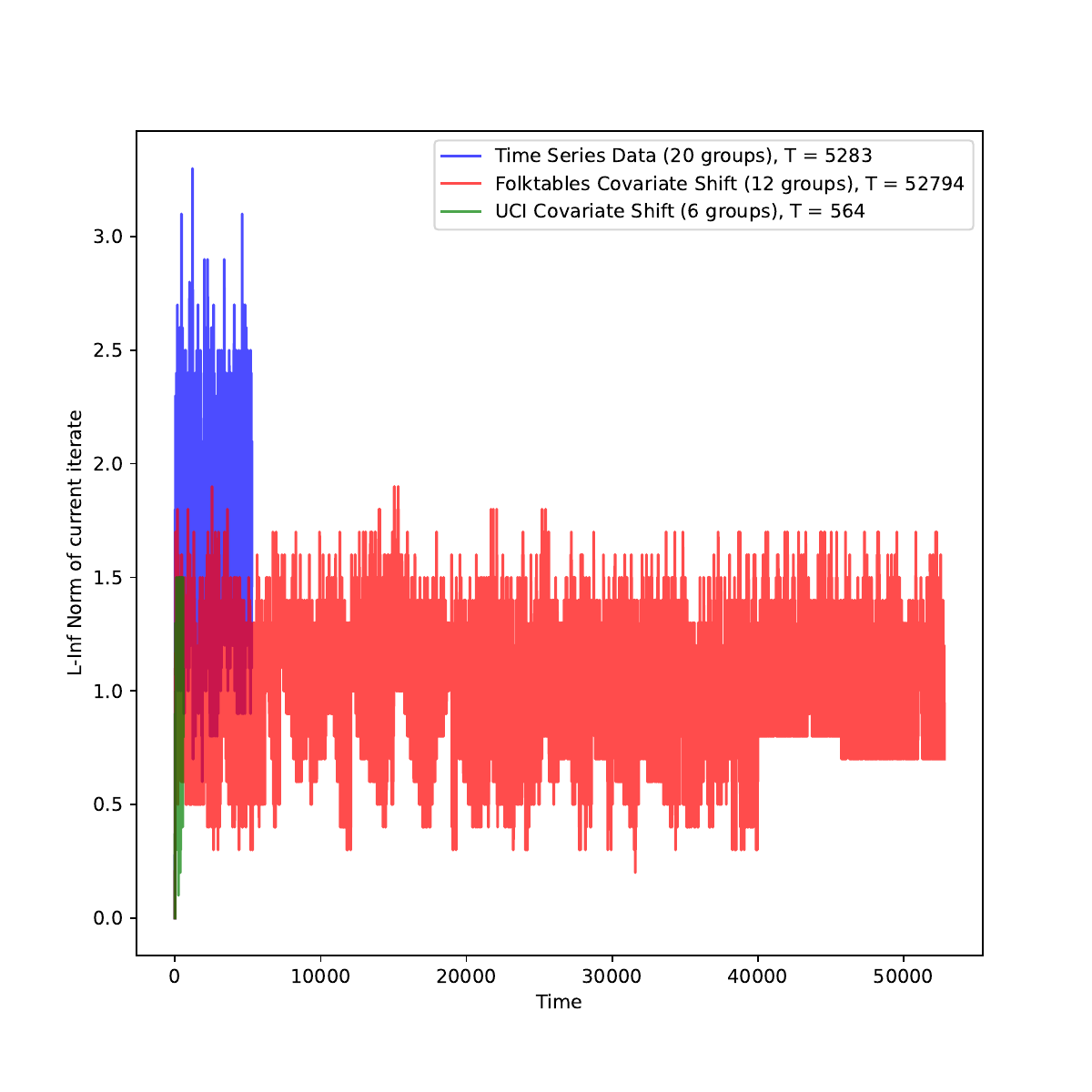} 
    \caption{$||\theta_t||_{\infty}$ over time for all three experiments, when running GCACI.}
    \label{fig:norms}
\end{figure}

\begin{figure*}[b!]
\begin{minipage}[b]{0.33\linewidth}
    \centering
    \includegraphics[width=\linewidth]{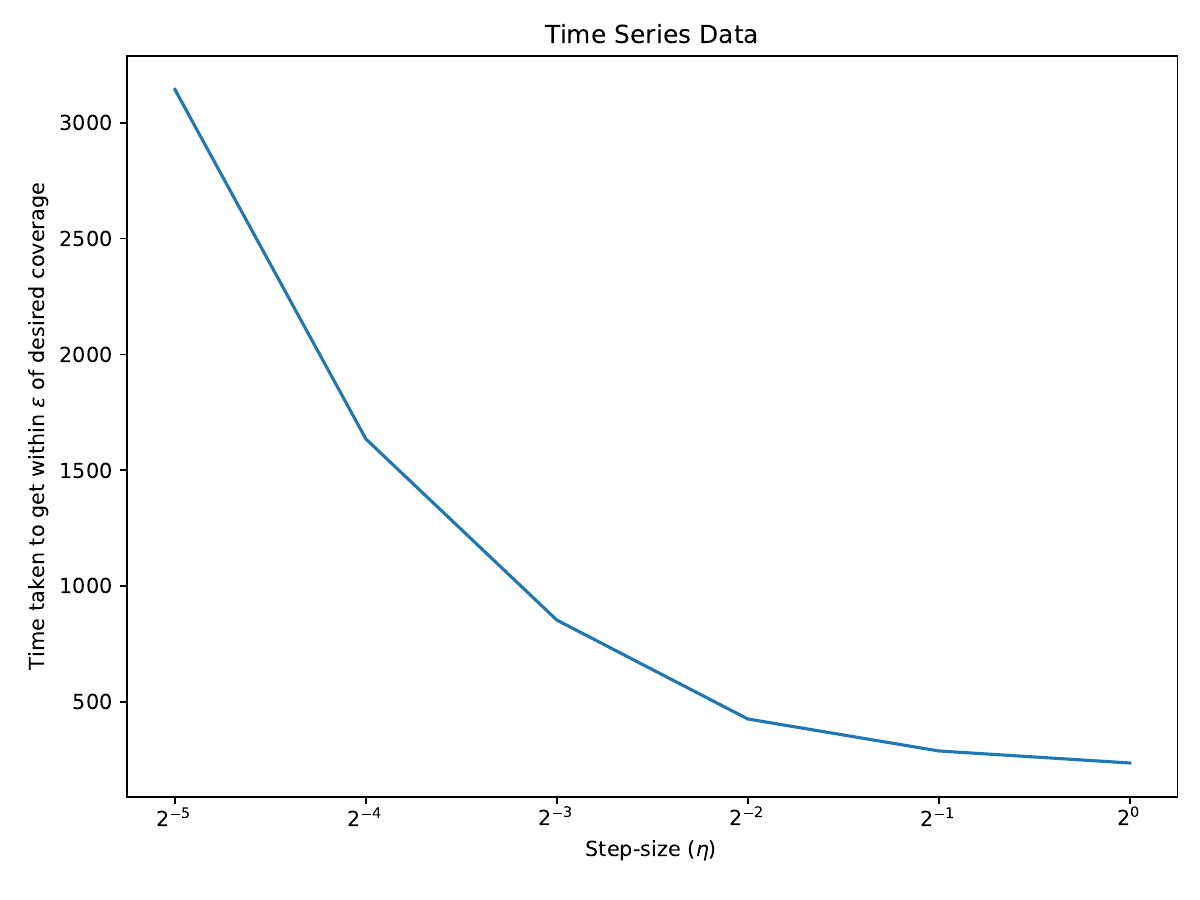}
    % \caption{Time Series Data}
\end{minipage}
\begin{minipage}[b]{0.33\linewidth}
    \centering
    \includegraphics[width=\linewidth]{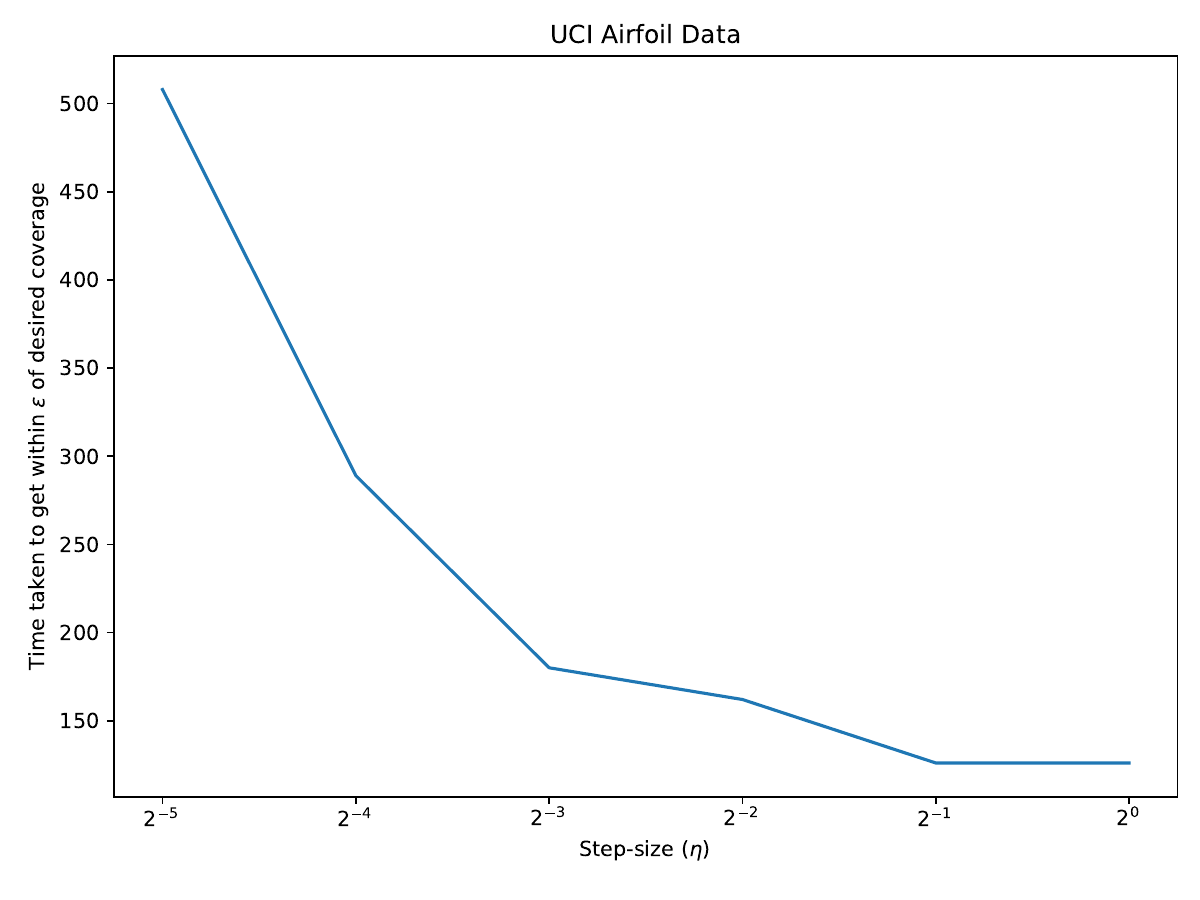}
    % \caption{UCI Data}
\end{minipage}
\begin{minipage}[b]{0.33\linewidth}
    \centering
    \includegraphics[width=\linewidth]{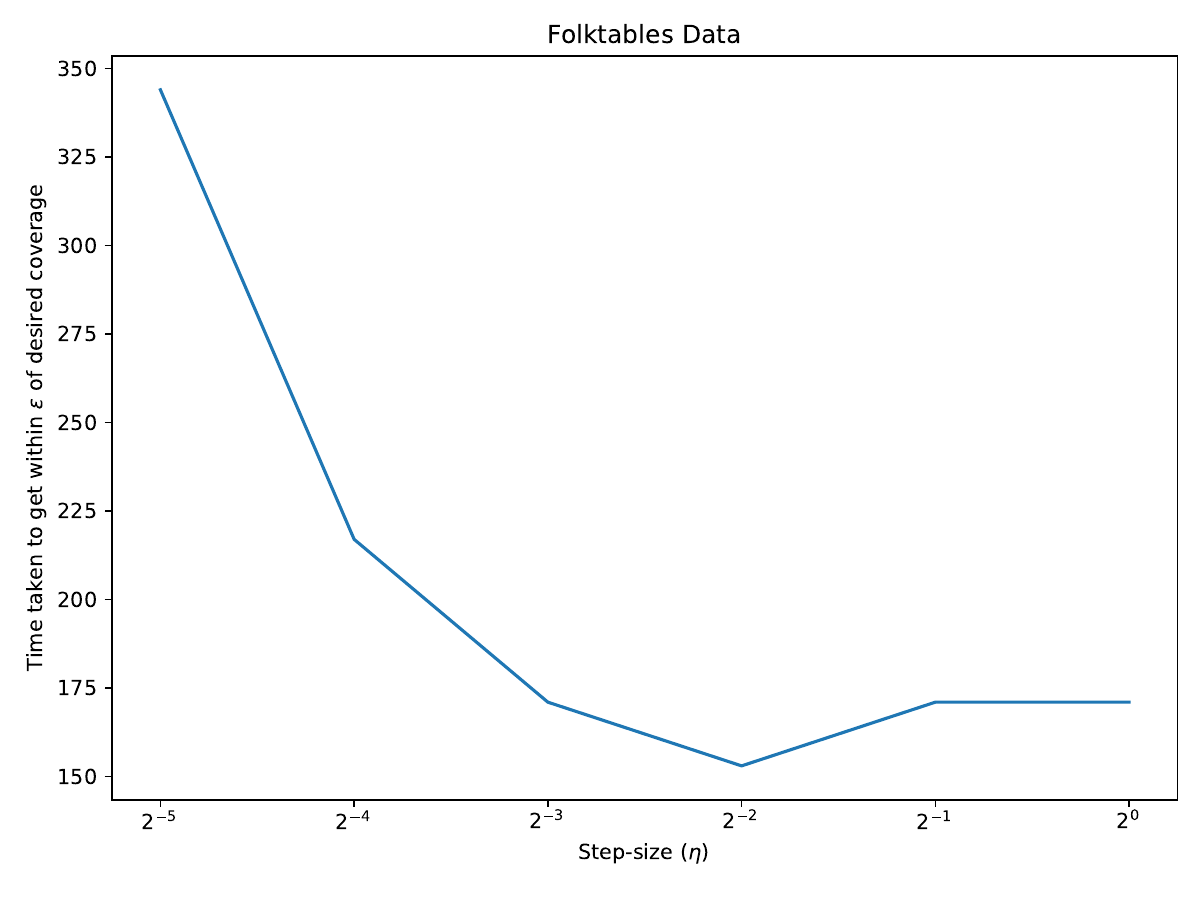}
        % \caption{Folktables Data}
\end{minipage}
\caption{Convergence rate of GCACI as a function of the learning rate $\eta$. Each plot measures (across different chosen learning rates) the earliest time-step at which coverage for each group is within $\epsilon = 0.01$ of the desired coverage for the rest of the transcript.}
\label{fig:learning-rate}
\end{figure*}

\paragraph{Natural Distribution Shift (Folktables)} Finally, we compare performance of MVP and GCACI on a distribution shift problem using 2018 Census data from the Folktables repository \cite{ding2021retiring}. The task involves predicting individuals' income. We use census data from two different states (California \& Pennsylvania) and sample 0.2 of both states to get a test set with $N = 52794$ data points. The data is sequenced with all CA datapoints first, giving us distribution shift from a natural source, going beyond covariate shift. A quantile regression model $h: \cX \to \mathbb{R}$ is trained on 50\% of the remaining California data, defining the fixed non-conformity score $f(x,y) = |h(x) - y|$. We define 12 total groups, over all nine codes for race\footnote{In \cite{bastanipractical}, four of the race groups are omitted due to being very small fractions of the overall dataset. We include even these small-sized groups to illustrate that GCACI is able to converge quickly even for such groups.} available in the Folktables dataset, two groups for sex, as well as the group including all data points. We run both algorithms with a desired coverage rate of $q = 0.9$. 

\subsection{Results}
\label{subsec:results}
Figure \ref{fig:convergence} compares how quickly the two algorithms are able to achieve the desired miscoverage rate. We see that convergence is substantially faster for our algorithm --- despite the fact that both algorithms have similar $O(\sqrt{T})$ guarantees for worst-case coverage rates. MVP doesn't even converge fully for some smaller-sized groups. We also find that for GCACI, the $O(\sqrt{T})$ upper bound on $||\theta_T||_\infty$ appears to be very loose, at least in the setting of our evaluation. Figure \ref{fig:norms} shows that for each experiment, it remains bounded by a small constant, explaining our superior observed coverage performance --- because in our experiments, $||\theta_t||_\infty$ remains bounded by a small constant at all iterates $t$, we actually get groupwise coverage rates at $O(1/T)$. This supports our conjecture that much better bounds might be possible for binary group structure. Figure \ref{fig:learning-rate} plots how quickly GCACI converges as a function of the learning rate. We see that as expected, larger learning rates give faster convergence, with the algorithm generally converging most quickly with a learning rate of $\eta = 1$. This naturally trades off with the regret guarantees of follow the regularized leader, which are optimized in the worst case when $\eta = 1/\sqrt{T}$ and vacuous for constant $\eta$. 

\bibliographystyle{plainnat}
\bibliography{refs}

\appendix
\onecolumn
\section{Proofs}
\iidloss*
\begin{proof}
Assume without loss of generality that $\tau' \leq \tau^*$. Define the probabilities $p_1 = \p(\tau \leq \tau')$, $p_2 = \p(\tau \geq \tau^*)$, and $p_3 = \p(\tau \in [\tau', \tau^*))$. We can compute the above expectation by looking at these three cases separately. When $\tau \leq \tau'$, the difference in loss is $(1-q)(
\tau' - \tau^*)$. Similarly, when $\tau \geq \tau^*$, the difference in loss is $q(\tau^* - \tau')$. Finally, when $\tau' \leq \tau \leq \tau^*$, 
\begin{align*}
   p_q(\tau', \tau) - p_q(\tau^*, \tau) &= q(\tau - \tau') - (1-q)(\tau^* - \tau) \\
   &= -q(\tau' - \tau^*) - (\tau^* - \tau) \\
   % &= (1 - q)(\tau' - \tau^*) - (\tau^* - \tau) - (\tau' - \tau^*)
   &= (1 - q)(\tau' - \tau^*) + (\tau - \tau')
   % & \leq (1 - q)(\tau' - \tau^*) + (\tau^* - \tau')
\end{align*}
Weighting each of these expectations with their respective probabilities, 
\begin{align*}
   \E[p_q(\tau', \tau)] - \E[p_q(\tau^*, \tau)] &= p_1 (1-q) (\tau' - \tau^*) - p_2(q)(\tau^* - \tau') + p_3((1 - q)(\tau' - \tau^*)  + (\tau - \tau')) \\
   &= p_3(\tau - \tau')
\end{align*}
with the final simplication due to $p_1 + p_3 = q$, and $p_2 = 1-q$, by definition of $\tau^*$. Since $\mathcal{D}$ is $(\alpha, \rho, r)$-smooth, we can obtain a lower-bound on $p_3$ by taking a discrete sum over $1/r$ pieces of the interval (each of which has probability weight at least $\alpha)$, to get:
\begin{equation*}
    \E[p_q(\tau', \tau)] - \E[p_q(\tau^*, \tau)] \geq \frac{\alpha r \cdot (\tau^* - \tau')^2}{2}
\end{equation*}
as desired. The proof for the $\tau^* \leq \tau'$ case is nearly identical.
\end{proof}

\stochcov*
\begin{proof}
    Define the realized loss $L= \sum_{t=1}^T p_q(\hat{\tau}_t, \tau_t)$ and the loss with respect to any fixed threshold $a$, as $L_a= \sum_{t=1}^T p_q(a, \tau_t)$, where $\tau_t = f(x_t, y_t)$. The regret guarantee tells us that
    \begin{equation*}
        L - L_{\tau^*} \leq \gamma
    \end{equation*}
    for $\tau^* = \min_{\tau \in [0,1]} \E[L_\tau]$ - this is the $q$-th quantile of the distribution $\mathcal{D}$. For $0 \leq t \leq T$, define the sequence of random variables $X_t = \mathbb{E}[L | \Pi_t]$, adapted to the filtration $\{\Pi_t: t \geq 0\}$. Note that since $\E[X_{t+1}| \Pi_t] = X_t$, this sequence is a martingale. Since $X_0 = \E[L]$ and $X_T = L$, using Azuma's inequality, gives us:
    \begin{equation*}
        \p[L - \E[L] \geq \epsilon] \leq \exp\left(-\frac{\epsilon^2}{2T}\right)
    \end{equation*}
    Thus we obtain a bound on the difference between expected losses: 
    \begin{equation*}
        \E[L] - \E[L_{\tau^*}] \leq \gamma + \epsilon
    \end{equation*}
    with probability at least $1 - 2\exp\left(-\frac{\epsilon^2}{2T}\right)$. Using Lemma \ref{lem:iid} separately for the difference in losses for each time-step, 
    \begin{equation}
        \label{eq:ineq}
         \sum_{t=1}^T\frac{\alpha r \cdot (\tau^* - \hat{\tau}_t)^2}{2} \leq \gamma + \epsilon \implies \sum_{t=1}^T (\tau^* - \tau_t)^2 \leq \frac{2(\gamma + \epsilon)}{\alpha r }
    \end{equation}
    Now, define for each round $t$ the expected miscoverage $M_t = \E_{\tau \in \mathcal{D}}[1[\hat{\tau}_t \geq \tau]] - q$. Since we know $\tau^*$ achieves the optimal coverage $q$, $M_t = \p(\tau \in [\hat{\tau}_t, \tau^*])$ (or the interval $[ \tau^*, \hat{\tau}_t]$), and due to the smoothness condition, this implies that 
    \begin{equation*}
        |\hat{\tau}_t - \tau^*| \geq \frac{M_t}{\rho r} \implies (\tau^* - \hat{\tau}_t)^2 \geq \frac{M_t^2}{\rho r}
    \end{equation*}
    Combining with the inequality from $(\ref{eq:ineq})$, we get:
    \begin{align*}
        \sum_{t=1}^T M_t^2 \leq \frac{2\rho(\gamma + \epsilon)}{\alpha} &\implies \sum_{t=1}^T M_t \leq \sqrt{T}\sqrt{\frac{2\rho(\gamma + \epsilon)}{\alpha}} \\
        &\implies \frac{1}{T}\sum_{t=1}^T M_t \leq \sqrt{\frac{2 \rho (\gamma + \epsilon)}{T \alpha}} 
    \end{align*}
    using Cauchy-Schwarz.  Another application of Azuma's inequality tells us that the average expected miscoverage above is more than $\epsilon/T$ away from the realized miscoverage rate with at most probability $2\exp\left(-\frac{\epsilon^2}{2T}\right)$. Taking a union bound over both probabilities, this gives us:
    \begin{equation*}
        |\texttt{Cov}(\Pi_T) - q| \leq \sqrt{\frac{2\rho(\gamma + \epsilon)}{T \alpha}} + \frac{\epsilon}{T}
    \end{equation*}
    with probability at least $1 - 4\exp\left(-\frac{\epsilon^2}{2T}\right)$.
\end{proof}

\discreteineq*
\begin{proof}
    Without loss of generality, assume that $a \leq b$. For any fixed $i \in [T]$, consider the difference $\Delta L_i = l_q(b, \tau_i) - l_q(a, \tau_i)$. There are three cases to consider. If $\tau_i < \min\{a, b\}$, then:
    \begin{equation*}
        \Delta L_i = (1 - q)(b - \tau_i - (a - \tau_i)) = (1 - q)(b- a)
    \end{equation*}
    Similarly, if $\max\{a,b\} \leq \tau_i$, then $\Delta L_i = q (a - b)$. For the third case, consider when $a \leq \tau_i < b$. Then, 
    \begin{align*}
        \Delta L_i = (1-q)(b-\tau_i) - q(\tau_i - a) &= q(a-b) + (b - \tau_i)
        % &\leq |(b - a)| + |b - a|  \\
        % &= (1+q)|b-a|
    \end{align*}
    Let $N_1, N_2$ and $N_3$ be the number of $i \in [t]$ falling into each of these three cases respectively. We first estimate $N_1$; since $a$ minimizes the sum of pinball losses, it must be one of the two grid-points $\mathcal{A}_n$ closest to the true $q$-th quantile of $\mathcal{D}$ (since the sum of pinball losses is a convex, piece-wise linear function). By the smoothness condition on $\mathcal{D}$, the amount of probability weight on this quantile cannot exceed $\rho$, and so we have $|N_1 - qT| \leq \rho T/2$. This implies that $|(N_2 + N_3) - (1-q)T| \leq \rho T/2$. Using $L_b - L_a \leq \gamma$ and writing this sum in terms of the above variables, 
    \begin{align*}
        \gamma & \geq N_1(1-q)(b-a) + qN_2(a-b) + qN_3(a-b) + \sum_{i: a \leq \tau_i < b} (b-\tau_i) \\
        &= (b-a)(N_1(1-q) - q(N_2 + N_3)) + \sum_{i: a \leq \tau_i < b} (b - \tau_i) \\
        &\geq \sum_{i: a \leq \tau_i < b} (b - \tau_i) 
    \end{align*}
    where the final inequality comes at the optimal values of $N_1 = qT, N_2 + N_3 = (1-q)T$. Using the smoothness condition on $\mathcal{D}$, we can lower-bound the sum by splitting the interval $[a,b]$ into pieces of length $1/r$, getting
    \begin{equation*}
        \sum_{i: a \leq \tau_i < b} (b - \tau_i) \geq \alpha \sum_{i=1}^{\lfloor r|b-a| \rfloor } \frac{i-1}{r} \geq \frac{T \alpha r (b-a)^2}{2} 
    \end{equation*}
    Rearranging, we get 
    \begin{equation*}
        (b - a)^2 \leq \frac{2\gamma}{T \alpha r} \implies |b - a| \leq \sqrt{\frac{2\gamma}{T \alpha r}}
    \end{equation*}
    % Therefore, for each $i \in [T]$, we have that:
    % \begin{equation}
    %     \Delta L_i \leq |a-b|(1+q)
    % \end{equation}
    % and since $L_b - L_a = \sum_{i=1}^T \Delta L_i$, applying the inequality above to each $i$ and rearranging gives us the result.
\end{proof}

\regrettocov*
\begin{proof}
    Since each predicted value $\hat{\tau}_t$ is in $\mathcal{A}_n$, we can rewrite regret via the separate contributions over each prediction value:
    \begin{align*}
        r(\Pi_t, -p_q, \phi) = \sum_{\tau \in \mathcal{A}_n} \underbrace{\sum_{t: \hat{\tau}_t = \tau} p_q(\hat{\tau}_t, \tau_t) - p_q(\phi(\hat{\tau_t}), \tau_t)}_{r_{\tau, \phi}}
    \end{align*}
    Define the swap function $\phi_m$ that, for each $\tau \in \mathcal{A}_n$, is defined as:
    \begin{equation*}
        \phi_m(\tau) = \min_{\tau' \in \mathcal{A}_n} \sum_{t: \hat{\tau} = \tau} p_q(\tau', \tau_t)
    \end{equation*}
    as well as the loss minimizer mapping $M: \mathcal{A}_n \to [0,1]$:
    \begin{equation*}
        M(\tau) = \min_{\tau' \in [0,1]} \sum_{t: \hat{\tau} = \tau} p_q(\tau', \tau_t)
    \end{equation*}
    Note that by definition, since $M(\tau)$ minimizes the sum of pinball losses, it is the $q$-th quantile of the empirical distribution $\mathcal{D}_{\tau}$ over the set $\{\tau_t\}_{t: \hat{\tau} = \tau}$. Further, since the sum of pinball losses (as a function of the first argument) is a convex, piece-wise linear function, $\phi_m(\tau)$ must be one of the two closest grid-points in $\mathcal{A}_n$ to $M(\tau)$, i.e. we have $|M(\tau) - \phi_m(\tau)| \leq 1/n$. 
    Since $r_{\tau, \phi} \geq 0$ for each $\tau \in \mathcal{A}_n$, a total swap-regret of $\gamma$ implies that $r_\tau \leq \gamma$ for each $\tau$. Using Lemma \ref{lem:discrete-ineq},
    \begin{equation*}
        |\phi_m(\tau) - \tau| \leq \sqrt{\frac{2\gamma}{T_\tau \alpha r}}
    \end{equation*}
    where we define $T_\tau = \sum_{t \in [T]} \one[\hat{\tau}_t = \tau]$. Due to the $(\rho, r)$-smoothness condition over $\mathcal{D}_{\tau}$, the amount of probability weight on $M(\tau)$ cannot exceed $\rho$, and so the number of values $N_{\tau}$ in $\{\tau_t\}_{t:\hat{\tau} = \tau}$ that $M(\tau)$ equals or exceeds satisfies $qT_{\tau} - \rho T_{\tau}/2 \leq N_{\tau} \leq qT_{\tau} + \rho T_{\tau}/2$. Finally, using the bound on $|M(\tau) - \tau|$ along with the smoothness condition, the number of values in the set $\{t:\hat{\tau} = \tau\}$ between $M(\tau)$ and $\tau$ cannot exceed $T_\tau \cdot \rho r \cdot\left(\frac{1}{n} + \sqrt{\frac{2\gamma}{T_\tau \alpha r}}\right)$. Using the upper bounds of the inequalities, 
    \begin{equation*}
        \sum_{t: \hat{\tau} = \tau} \one[\tau_t \leq \tau] \leq  qT_{\tau} + \frac{\rho T_{\tau}}{2} + \frac{\rho r T_\tau}{n} + \sqrt{\frac{2\gamma T_\tau}{\alpha r}}
    \end{equation*}
    Notice that the left hand side equals $\texttt{Cov}(\Pi_T, G_\tau)$, where $G_\tau$ is the binary group including all time-steps $t$ for which $\hat{\tau}_t = \tau$. Thus, performing the same steps using the lower bounds of the inequality and dividing by $T_\tau$,
    \begin{equation*}
        \left| \texttt{Cov}(\Pi_T, G_\tau) - q\right| \leq \frac{\rho}{2} + \frac{\rho r}{n} +\sqrt{\frac{2\gamma}{T_\tau \alpha r}}
    \end{equation*}
    % Summing this inequality across all $\tau \in \mathcal{A}_n$ (that have at least length 1), we get:
    % \begin{align*}
    %     T \cdot \texttt{Cov}(\Pi_T) &\leq qT + \frac{\rho T}{2} + \frac{\rho r T}{n} + \sum_{\tau \in \mathcal{A}_n } \sqrt{\frac{2\gamma}{T_\tau \alpha r}} \\
    %     &\leq qT + \frac{\rho T}{2} + \frac{\rho r T}{n} + n \sqrt{\frac{2\gamma}{\alpha r}}
    % \end{align*}
    % Performing the same steps using the lower bounds of the inequality and dividing over by $T$, 
    % \begin{equation*}
    %     \left| \texttt{Cov}(\Pi_T) - q\right| \leq \frac{\rho}{2} + \frac{\rho r}{n} +  \frac{n}{T}\sqrt{\frac{2\gamma}{\alpha r}}
    % \end{equation*}
\end{proof} 

\covtoregret*
\begin{proof}
    Fix a threshold $\tau \in \mathcal{A}_n$. Let $M(\tau) = \min_{a \in [0,1]}| \texttt{Cov}(\Pi_T, G_{\tau}) - q|$ where $G_\tau$ is the binary group including all time-steps for which the predicted threshold was $\tau$. Note that since by definition $M(\tau)$ is the $q$-th quantile of the empirical distribution $\mathcal{D}_\tau$, it is also the value $a$ that minimizes the sum of pinball losses $\sum_{t=1}^T \one[\hat{\tau}_t = \tau] \cdot p_q(a, \tau_t)$. We assume without loss of generality that this exact $q$-th quantile over $\mathcal{D}_\tau$ exists - since any other value would get worse miscoverage (with respect to the desired rate $q$), $\tau$ achieving comparable performance to the true minimizer implies it would achieve the same (or better) performance with respect to $M(\tau)$ even if the exact $q$-th quantile did not exist. By the coverage error guarantee, $$\p_{\tau \in \mathcal{D}_\tau}(\tau \in [M(\tau),  \tau]) \leq \gamma$$ or instead $[\tau, M(\tau)]$, based on their ordering.  Using the smoothness condition, we have:
    \begin{equation*}
        |\tau - M(\tau)| \leq \frac{\gamma}{\alpha r}
    \end{equation*}
    Assume without loss of generality that $M(\tau) \leq \tau$. Define the variables $N_1, N_2$ and $N_3$ as the number of thresholds in the set $\{\tau_t\}_{t: \hat{\tau}_t = \tau}$ less than or equal to $M(\tau)$, in the interval $(M(\tau), \tau]$, and greater than $\tau$ respectively. Since $M(\tau)$ is exactly the $q$-th quantile, $N_1 = qT_\tau$ and $N_2 + N_3 = (1-q)T_\tau$, where we define $T_\tau = \sum_{t=1}^T \one[\hat{\tau_t} = \tau]$. We can rewrite the difference in pinball loss by dividing into these three categories, as in the proof for Lemma $\ref{lem:discrete-ineq}$, to get:
    \begin{align*}
        \sum_{t: \hat{\tau}_t = \tau} p_q(\tau, \tau_t) - p_q(M(\tau), \tau_t) &= N_1(1-q)(\tau - M(\tau)) - qN_2(\tau - M(\tau)) + qN_3(\tau - M(\tau))+ \sum_{i: a \leq \tau_i < b, \hat{\tau}_i = \tau} (\tau - \tau_i) \\
        &= (\tau - M(\tau))(N_1(1-q) - q(N_2 + N_3)) + \sum_{i: M(\tau) \leq \tau_i < \tau, \hat{\tau}_i = \tau} (\tau - \tau_i) \\
        &\leq N_3 (\tau - M(\tau)) \leq T_\tau \frac{\gamma \rho r}{\alpha r} \cdot \frac{\gamma}{\alpha r} = T_\tau \frac{\gamma^2 \rho}{\alpha^2 r}
    \end{align*}
    using the smoothness condition and the bound on $|M(\tau) - \tau|$ to bound the value of $N_3$. Thus the maximal regret with respect to the best action in hindsight $M(\tau)$ over the subsequence where only prediction $\tau$ is made can be bound:
    \begin{align*}
        r_{\tau} &= \max_{a \in \mathcal{A}_n} \sum_{t: \hat{\tau}_t = \tau} p_q(\tau, \tau_t) - p_q(a, \tau_t) \\
        &\leq \sum_{t: \hat{\tau}_t = \tau} p_q(\tau, \tau_t) - p_q(M(\tau), \tau_t) \leq  T_\tau \frac{\gamma^2 \rho}{\alpha^2 r}
    \end{align*}
    The set of subsequences defined by $G_\tau$ (across all $\tau \in \mathcal{A}_n$) forms a partion over the full transcript; summing the above inequality for $r_\tau$ across all $\tau \in \mathcal{A}_n$, 
    \begin{equation*}
        r(\Pi_t, -p_q, \phi) \leq \frac{T \gamma^2 \rho}{\alpha^2 r}
    \end{equation*}
\end{proof}

% \FTRL*

% \groupcov*

% \iterateval*

\counterexample*
\begin{proof}
Assume that $g_1 = 0$. Since $\tau_t = 1$, whenever $\theta_t \cdot g_t < 1$ we will trigger update A. Assume all rounds up through round $t-1$ triggered update $A$, in which case  $\theta_{t} = \eta \cdot q \cdot \sum_{k=1}^{t-1} g_k < 2\sqrt{t-1}$. But because we set $g_t = \frac{1}{2\sqrt{t-1}}$ we have that $\theta_t \cdot g_t < 1$, once again triggering update $A$. Inductively, update $A$ is thus triggered at every round, and so we have that  $\theta_{T+1} = \eta \cdot q \cdot \sum_{k=1}^{T} g_k = \Omega(\sqrt{T})$.
\end{proof}

\end{document}